\documentclass[journal]{IEEEtran}

\usepackage{amssymb}
\usepackage{amsfonts}
\usepackage{amsmath}
\usepackage{graphicx}
\usepackage{graphics}
\usepackage{cite}
\usepackage{color}
\usepackage{multirow}
\usepackage{boxedminipage}
\usepackage{booktabs}
\usepackage{algorithm, algpseudocode}

\usepackage{amsthm}
\newtheorem{theorem}{Theorem}

\newtheorem{assumption}{Assumption}

\input{mysymbol.sty}

\begin{document}

\title{Multilinear Tensor Low-Rank Approximation for Policy-Gradient Methods in Reinforcement Learning}

\author{Sergio~Rozada,~\IEEEmembership{Student Member,~IEEE,} Hoi-To~Wai,~\IEEEmembership{Member,~IEEE,}
        and~Antonio~G.~Marques,~\IEEEmembership{Senior Member,~IEEE}

\thanks{S. Rozada and A. G. Marques are with the Department
of Signal Theory and Comms., King Juan Carlos University, Madrid, Spain, s.rozada.2019@alumnos.urjc.es, antonio.garcia.marques@urjc.es. HT Wai is with the Department of Systems Eng. and Eng. Management, in The Chinese University of Hong Kong (CUHK), Hong Kong, China, htwai@cuhk.edu.hk.}%

\thanks{Work partially funded by the Spanish NSF (MCIN/AEI/10.13039/ 501100011033) Grants PID2019-105032GB-I00, TED2021-130347B-I00, and PID2022-136887NB-I00, the “European Union NextGenerationEU/PRTR”, and the Community of Madrid via the Ellis Madrid Unit and grant TEC-2024/COM-89.}}

\maketitle


\begin{abstract}
Reinforcement learning (RL) aims to estimate the action to take given a (time-varying) state, with the goal of maximizing a cumulative reward function. Predominantly, there are two families of algorithms to solve RL problems: value-based and policy-based methods, with the latter designed to learn a probabilistic parametric policy from states to actions. Most contemporary approaches implement this policy using a neural network (NN). However, NNs usually face issues related to convergence, architectural suitability, hyper-parameter selection, and underutilization of the redundancies of the state-action representations (e.g. locally similar states). This paper postulates multi-linear mappings to efficiently estimate the parameters of the RL policy. More precisely, we leverage the PARAFAC decomposition to design \emph{tensor low-rank} policies. The key idea involves collecting the policy parameters into a tensor and leveraging tensor-completion techniques to enforce  \emph{low rank}. We establish theoretical guarantees of the proposed methods for various policy classes and validate their efficacy through numerical experiments. Specifically, we demonstrate that \emph{tensor low-rank} policy models reduce computational and sample complexities in comparison to NN models while achieving similar rewards.
\end{abstract}


\section{Introduction}

We inhabit a world characterized by growing complexity and interdependence, where gadgets produce vast streams of data on a daily basis. Computer systems need to become intelligent to make decisions autonomously in order to navigate such complex landscapes with little human supervision. Among the various frameworks suggested for developing autonomous intelligent systems, Reinforcement Learning (RL) has emerged as particularly promising. Broadly speaking, software agents in RL learn by interacting with the world by trial and error \cite{sutton2018reinforcement, bertsekas2019reinforcement}. RL has contributed significantly to pivotal achievements within the Artificial Intelligence (AI) community. Notable examples of the latter include AlphaGo, a computer software that beat the world champion in the game of Go \cite{silver2016mastering, silver2017mastering}, and recently, groundbreaking chatbots like ChatGPT \cite{brown2020language}.

From a technical standpoint, RL is rooted in stochastic approximation for Dynamic Programming (DP). RL solves a sequential optimization problem, aiming to maximize a reward signal over time \cite{bertsekas2000dynamic, bertsekas1996neuro}. In this framework, the world, or environment, is typically represented as a Markov Decision Process (MDP) \cite{sutton2018reinforcement}. Succinctly, the environment consists of a set of states, or state space, where the agents take actions from a set of actions, or action space. Notably, the states and the actions are usually multi-dimensional. Unlike DP, RL does not assume knowledge of the underlying dynamics of the MDP model, with learning taking place from sampled trajectories. In a nutshell, the goal of RL is to learn a \emph{state-action map}, or \emph{policy}, that maximizes the expected cumulative reward from samples of the MDP. Originally, value-based methods proposed to learn the optimal value functions (VFs) within the MDP first, to then infer the optimal policy \cite{sutton2018reinforcement, bertsekas2019reinforcement}. Yet, value-based methods face problems due to the curse of dimensionality. An alternative to overcome this limitation is offered by policy-based RL, which proposes to learn directly a policy, which almost always is assumed to be probabilistic and parametric \cite{sutton1999policy}. Therefore, the problem shifts from learning the VFs to estimate the parameters of a probabilistic mapping. Typically, Gaussian distributions are employed for continuous action spaces, and softmax distributions for discrete action spaces \cite{paternain2020stochastic, chen2021communication, luo2017learning}.

The conventional approach to solve policy-based RL problems involves learning a mapping from states to the \emph{parameters of the probability distribution} via gradient-based algorithms, also known as policy-gradient methods (PGs) \cite{arulkumaran2017deep}. Designing policy maps involves balancing expressivity and tractability. Initially, linear models were favored for their simplicity and ease of understanding, but they lacked expressive power and required significant feature engineering. More recently, neural networks (NNs) have emerged as an alternative to handle more complex mapping tasks. However, finding the adequate NN architecture is often a challenge on its own. A poor selection of the architecture may lead to convergence issues and moreover, even when the architecture seems fitting, NN training can be time-consuming. To narrow down the space of potential models and to better condition the problem, a practical strategy involves postulating parsimonious models that capitalize on the problem's inherent structure. Notably, leveraging sparsity has been a widely adopted tool to design policy-based algorithms \cite{lever2016compressed, tolstaya2018nonparametric, hu2022ddpg}. Moreover, low-rank has been recently explored in the context of value-based RL to enforce low rank in matrix and tensor representations of the VFs \cite{yang2019harnessing, tsai2021tensor, rozada2024tensor}. Intriguingly, little research has been carried out on low-rank models for policy-based RL \cite{rozada2023matrix}. This gap is the main driver of this paper, that puts forth low-rank policy models that give rise to \emph{efficient} and \emph{fast} policy-based RL algorithms. Given the potentially high dimensionality of the state space, we propose to collect the state-dependent parameters of the policy probability distribution in a tensor of parameters, to then enforce low rank via tensor completion techniques. This gives rise to multi-linear \emph{tensor low-rank} policy models can be seamlessly integrated into classical policy-based RL algorithms, as we will show in the following sections.

Low-rank methods have been widely applied in optimization, notably in the design of matrix completion algorithms \cite{eckart1936approximation, markovsky2012low, udell2016generalized}. In the realm of tensors, the multidimensional generalization of matrices, low rank has also found success in developing resource-efficient optimization methods \cite{kolda2009tensor, sidiropoulos2017tensor}. Interestingly, MDPs can be low rank, meaning that transition and reward dynamics are low rank \cite{agarwal2020flambe, uehara2021representation}. Low-rankness yields the creation of efficient model-based algorithms, where enforcing low rank in a matrix or a tensor representation of the transition dynamics has been instrumental \cite{barreto2016incremental, jiang2017contextual, azizzadenesheli2016reinforcement, mahajan2021tesseract}. Moreover, recent empirical studies have revealed that matrix and tensor representations of VFs are oftentimes low rank \cite{yang2019harnessing, rozada2024tensor}. This has led to value-based methods that reduce sample complexity \cite{sam2023overcoming, shah2020sample} and showed faster empirical convergence rates \cite{yang2019harnessing, tsai2021tensor, rozada2021low, rozada2024tensor}. While the exploration of low rank in policy-based RL has not been as extensive, recent developments introduced matrix-completion techniques to design Gaussian policy models in actor-critic (AC) methods and in trust-region policy optimization (TRPO) \cite{rozada2023matrix, rozada2023trust}. However, matrix-based policies do not scale well to large state-action spaces, and as action spaces are not always continuous, Gaussian policies can often be a poor choice.

Motivated by these discoveries, this paper aims to present a comprehensive methodology for crafting low-rank policy-based RL algorithms. Our specific contributions include:

\begin{itemize}
    \item[\textbf{C1}] Policy \emph{models} that enforce low rank in a tensor of Gaussian-policy parameters and softmax-policy parameters using the PARAFAC decomposition. This is equivalent to postulate a multilinear model for the policy.
    \item[\textbf{C2}] Policy-based \emph{algorithms}, including REINFORCE, AC, TRPO and PPO, leveraging tensor low-rank (multilinear) policy models.
    \item[\textbf{C3}] Convergence guarantees for the proposed tensor low-rank (multilinear) PG algorithms.
    \item[\textbf{C4}] \emph{Numerical experiments} illustrating that the proposed low-rank methods converge faster than NN-based schemes.
\end{itemize}

\section{Preliminaries}

The objective of this section is to introduce notation, cover the basics of RL, and provide a concise overview of tensor low-rank methodologies.

\subsection{Reinforcement Learning}

RL is a framework to solve sequential optimization problems in Markovian setups \cite{sutton2018reinforcement, bertsekas2019reinforcement}. Not surprisingly, the main mathematical formalism is the Markov Decision Process (MDP), which defines the environment as a set of states $\ccalS$ and a set of actions $\ccalA$. The dynamics of the MDP are encoded in the probability transition function $p \left( s' | s, a \right)$, which defines the probability of transitioning from state $s \in \ccalS$ to state $s' \in \ccalS$ after taking action $a \in \ccalA$. The initial state is distributed according to $p_1$. The reward function $r \left( s, a, s' \right)$ defines the instantaneous reward associated with being in the state $s \in \ccalS$, taking the action $a \in \ccalA$, and reaching the new state $s' \in \ccalS$. Interestingly, the reward function $r$ can be deterministic or stochastic. A key aspect of RL is that we do not have access to $p$, $p_1$, and $r$, often referred to as the model of the MDP.

The fundamental goal of RL is to learn the optimal sequence of actions to maximize the long-termed reward. In policy-based RL we aim to learn a probabilistic parametric policy $\pi_\theta$ that maximizes the cumulative reward {over a period of $T$ time steps}. {Let $\theta \in \mathbb{R}^p$ be a policy parameter vector, the policy} $\pi_\theta \left( a | s \right)$ is a map that defines the conditional probability of taking the action $a \in \ccalA$ being in state $s \in \ccalS$. As the model of the MDP is unknown, the policy $\pi_\theta$ is learned from trajectories sampled from the MDP. Let $t=1,...,T$ be a time index. The trajectory $\tau_T = \{ (S_t, A_t) \}_{t=1}^{T+1} \sim g (\theta)$ 
is a Markov {chain generated by} $S_1 \sim p_1$, and for any $t \geq 1$, $A_t \sim \pi_\theta( \cdot | S_t )$, $S_{t+1} \sim p( \cdot | S_t, A_t )$. The reward $R_t = r(S_t, A_t, S_{t+1})$ is a random variable that evaluates the \emph{immediate} value of the selected state-action pair at time-step $t$. However, the action $A_t$ influences $S_{t'}$ for $t'>t$, thereby impacting subsequent rewards $R_{t'}$ for $t'>t$. This illustrates that the problem is coupled across time and the cumulative reward, or return, of a trajectory $G_T( \tau_T ) = \sum_{t=1}^T R_t$ must be considered in the optimization. 

Subsequently, the policy optimization problem can be written as
\begin{equation} \label{eq::pg_cost}
    \max_{ \theta \in \mathbb{R}^p } ~\bar{G}_V(\theta) := \mathbb{E}_{ \tau_T \sim g(\theta) } [ G_T(\tau_T) ],
\end{equation}
where the decision variable appears in the distribution that leads to the sampled trajectory $\tau_T$ and the problem is non-convex in general.  
To see this clearer, consider a simple setup involving actions modeled as samples of a univariate Gaussian policy as $A_t \sim \ccalN (\mu_{S_t}, \sigma_{S_t})$, with policy parameters $\theta_{\mu} = \{ (\mu_{s}) \}_{s \in \ccalS}$ and $ \theta_{\sigma} = \{ (\sigma_{s})\}_{s \in \ccalS}$.
Clearly, the dimension of $\theta$ depends on the state space $\ccalS$ that can become infinity for continuous state space.
The dimensionality problem is usually addressed via parametric mappings that estimate the parameters of the probability distribution associated with the policy, e.g.  $\mu_{\theta_\mu}, \sigma_{\theta_\sigma}: \ccalS \mapsto \reals$ in the Gaussian case.
The problem \eqref{eq::pg_cost} evaluates the expected returns over trajectories, that are collections of transitions. For each transition, at a given state $s$, the action $a$ is sampled from the Gaussian distribution $\ccalN(\mu_s, \sigma_s)$, where the mean and the standard deviations are given by $\mu_s := \mu_{\theta_\mu}(s)$ and  $\sigma_s := \sigma_{\theta_\sigma}(s)$.
Next, we overview several common strategies for optimizing \eqref{eq::pg_cost}.

\noindent \textbf{Policy-gradient methods.} Directly optimizing \eqref{eq::pg_cost} entails using the density function $g(\theta)$ to compute an expectation over the trajectories. However, $g(\theta)$ relies on the unknown probability transition function $p$ and the initial state distribution $p_1$. As introduced previously, stochastic {approximation} methods are commonly used to circumvent this issue by approximating the expectation using samples $\tau_T \sim g (\theta)$ from the MDP. The workhorse of policy-based RL is the gradient-based algorithms that update the policy parameter vector $\theta$ through (stochastic) gradient ascent \cite{lee2020optimization}. Consequently, the update rule takes the form $\theta^{h+1} = \theta^h + \eta^h \nabla_\theta \bar{G}_V(\theta^h)$, where $h$ denotes the iteration index, and $\eta^h$ represents the step size or learning rate. 
The policy gradient theorem \cite{sutton1999policy} shows that the gradient of $\bar{G}_V(\theta)$ can be evaluated as
\begin{equation}
    \label{eq::pg_gradient}
    \nabla_\theta \bar{G}_V(\theta) = \mathbb{E}_{g \left( \theta \right)} \Big[\sum_{t=1}^T {G_T(\tau_T)} \; \nabla_\theta \text{log} \pi_\theta(A_t | S_t) \Big] .
\end{equation}
Interestingly, these schemes receive the name of policy gradient (PG) methods, as computing the gradient of $\bar{G}_V$ boils down to computing the gradient of the log-policy $\nabla_\theta \text{log}\;\pi_\theta$ with respect to (w.r.t.) the parameters $\theta$. The term $\nabla_\theta \text{log}\;\pi_\theta$ is also known as the \emph{policy score}. A celebrated PG method is Monte-Carlo PG (a.k.a.~REINFORCE) \cite{williams1992simple}, which samples trajectories from the MDP to update the parameters via \eqref{eq::pg_gradient}, using the sample return $G_t=\sum_{t'=t}^T R_{t'}$.

\vspace{.5mm}
\noindent \textbf{Actor-critic methods.} Monte-Carlo PG methods suffer from the credit assignment problem, 
i.e., determining the contribution of each action becomes challenging when the return is observed only at the end of the episode.
This often leads to high-variance issues \cite{schulman2015high}. A common strategy to mitigate the high variance in the gradient estimate while maintaining the bias unchanged involves subtracting a baseline value from the return $G_t$ \cite{greensmith2004variance}. The most common baseline is the VF, defined as $\mathbb{V}^{\pi_\theta}(S_t) = \mathbb{E}_{\pi_\theta} \left[ G_t | S_t \right]$. This results in an estimator of the advantage function \cite{schulman2015high} as $\mathbb{A}^{\pi_\theta}(S_t) = G_t - \mathbb{V}^{\pi_\theta}(S_t)$. However, computing the exact VFs associated with the policy $\pi_\theta$ is a challenge by itself, usually addressed by learning a parametric approximation $\mathbb{V}_\omega$. This leads to AC methods, that operate in two phases, by estimating first the parameters of the policy $\pi_\theta$ or actor, and then the parameters of the VF approximator $\mathbb{V}_\omega$, or critic \cite{konda1999actor}. The AC counterpart of \eqref{eq::pg_gradient} is
\begin{equation}
    \label{eq::pg_ac_gradient}
    \nabla_\theta \bar{G}_A(\theta) = \mathbb{E}_{g \left( \theta \right)} \Big[ \sum_{t=1}^T \mathbb{A}_\omega (S_t) \nabla_\theta \text{log}\pi_\theta(A_t | S_t) \Big],
\end{equation}
\noindent where $\mathbb{A}_\omega (S_t) = G_t - \mathbb{V}_\omega(S_t)$ is the advantage estimator, with $\mathbb{V}_\omega(S_t)$ fixed. Stochastic gradient updates of the actor reduces again to compute the policy scores $\nabla_\theta \text{log}\pi_\theta$. The critic parameters $\omega$ are obtained via the minimization of the mean-squared error $L(\omega)=\frac{1}{2}\sum_{t=1}^T (G_t - \mathbb{V}_\omega(S_t))^2$, typically using gradient descent methods of the form $\omega^{h+1}=\omega^{h} - \alpha^h \nabla_\omega L(\omega^h)$, where $h$ is again the iteration index and $\alpha^h$ is the gradient step.

\vspace{.5mm}
\noindent \textbf{Trust-region methods.} Interestingly, monotonic improvements of the policy can be guaranteed by constraining the policy updates within small trust-regions \cite{kakade2001natural, kakade2002approximately, schulman2015trust}. This fact motivated trust-region policy optimization (TRPO) \cite{schulman2015trust}, a method that proposes reformulating the actor update in the AC problem to find a new policy $\pi_\theta$ in a trust-region of the actual (old) policy $\pi_{\tilde{\theta}}$ used to sample from the environment. This region is defined by the KL-divergence $D_{KL}(\cdot \| \cdot)$. More formally, we denote the probability ratio between the new policy $\pi_\theta$ and the (old) policy $\pi_{\tilde{\theta}}$ as 
\begin{equation}
    q_{\theta, \tilde{\theta}} (S_t, A_t) = \frac{\pi_\theta (A_t|S_t)}{\pi_{\tilde{\theta}}(A_t|S_t)}.
\end{equation}
Then, TRPO proposes to iteratively maximize:
\begin{equation}
    \label{eq::trpo_cost}
    \bar{G}_{T}(\theta) = \mathbb{E}_{ g (\tilde{\theta}) } \left[ \sum_{t=1}^T q_{\theta, \tilde{\theta}} (S_t, A_t) \mathbb{A}_\omega (S_t) \right],
\end{equation}
subject to $\mathbb{E}\left[D_{KL}(\pi_{\tilde{\theta}}(\cdot|s) || \pi_\theta(\cdot|s))\right] \le \delta$, where $\mathbb{E}\left[ \cdot \right]$ denotes the expectation over the states, and $\delta$ is the size of the neighborhood. TRPO involves i) sampling the MDP following $\pi_{\tilde{\theta}}$ to approximate the expectation $\mathbb{E}_{ g (\tilde{\theta}) } \left[ \cdot \right]$, and ii) maximize \eqref{eq::trpo_cost} using the samples, usually via second order optimization methods \cite{schulman2015trust}. In setups where second-order methods are computationally demanding, proximal policy optimization (PPO) \cite{schulman2017proximal} proposes constraining the policy updates using a clipping mechanism to design first-order TR methods. More specifically, PPO defines a trust-region around  the actual (old) policy $\pi_{\tilde{\theta}}$ using the clipped probability ratio
\begin{equation}
    \label{eq::ratio_ppo}
    \hat{q}_{\theta, \tilde{\theta}} (S_t, A_t) = 
            \begin{cases}
              1 - \epsilon & \text{if} \; q_{\theta, \tilde{\theta}} (S_t, A_t) \leq 1 - \epsilon  \\
              1 + \epsilon & \text{if} \; q_{\theta, \tilde{\theta}} (S_t, A_t) \geq 1 + \epsilon  \\
              q_{\theta, \tilde{\theta}} (S_t, A_t) & \text{otherwise},
            \end{cases}
\end{equation}
where $\epsilon$ defines the size of the trust-region, to then maximize the following objective
\begin{align}
    \label{eq::ppo_cost}
    \bar{G}_P(\theta) =& \mathbb{E}_{g \left( \tilde{\theta} \right)} \Big[ \sum_{t=1}^T \min \big\{ q_{\theta, \tilde{\theta}} (S_t, A_t) \mathbb{A}_\omega (S_t),\nonumber \\
     &\hspace{2.4cm}\; \hat{q}_{\theta, \tilde{\theta}} (S_t, A_t) \mathbb{A}_\omega (S_t)\big\} \Big].
\end{align}
Note that the objective in \eqref{eq::ppo_cost} clips the probability ratio to $1 + \epsilon$ for positive advantages and to $1 - \epsilon$ for negative advantages.
As we will show in subsequent sections, the actor update in TR methods (TRPO and PPO) boils down to compute the policy scores $\nabla_\theta \text{log}\pi_\theta$, as in PG and AC methods. This is a fundamental aspect of policy-based algorithms, since it implies that postulating new policy models can be ``decoupled'' from the algorithmic design. For the problems investigated in this paper, this implies that the \emph{low-rank multi-linear} policy models that we postulate in subsequent sections can be applied to a wide range of policy-based RL algorithms, provided that those algorithms are modified to use the policy scores proposed in this paper. 

\subsection{Low Rank and Tensor Completion}

Since tensors are a central piece of this work, we briefly review related basic concepts and notation. In data science contexts, tensors can be understood as a multi-dimensional generalization of matrices. From a formal standpoint, a tensor $\tenbX \in \reals^{N_1 \times \hdots \times N_D}$ is a $D$-dimensional array indexed by a $D$-dimensional index $\bbi = [i_1, \hdots, i_D]$ with $i_d\in\{1,...,N_d\}$. In tensor jargon, dimensions also receive the name of modes. In contrast to matrices, there are different definitions for the rank of the tensor. In this work, we focus on the PARAFAC rank, which has been widely used to postulate parsimonious data representations. Succinctly, the PARAFAC rank is defined as the minimum number of rank-1 tensors that, when aggregated, recover the original tensor \cite{kolda2009tensor, sidiropoulos2017tensor}. To be more precise, a tensor $\tenbZ$ has rank-1 if it can be written as the outer product of $D$ vectors. This is denoted as $\tenbZ=\bbz_1 \circledcirc \hdots \circledcirc \bbz_D$ and implies that the entry $\bbi$ of rank-1 tensor is simply given by $[\tenbZ]_{\bbi}=\prod_{d=1}^D[\bbz_1]_{i_d}$. Then, $\tenbX$ has a rank $K$ if it can be decomposed as $\tenbX = \sum_{k=1}^K  \bbx_1^k \circledcirc \hdots \circledcirc \bbx_D^k$, with the PARAFAC decomposition \cite{bro1997parafac} being the procedure that, given $\tenbX$, finds the vectors $\{\bbx_d^k\}_{d=1,k=1}^{D,K}$, known as factors. The factor vectors are oftentimes collected into matrices $\bbX_d=[\bbx_d^1,...,\bbx_d^K]$. Using the entry-wise definition of the outer product, it readily follows that entry of a tensor of rank $K$ is given by
\begin{equation}
    \label{eq::parafac_model}
    [\tenbX]_\bbi = \sum_{k=1}^K \prod_{d=1}^D [\bbx_d^k]_{i_d} = \sum_{k=1}^K \prod_{d=1}^D [\bbX_d]_{{i_d}, k},    
\end{equation}
revealing that assuming that the tensor has low-rank implies modeling the tensor as a multi-linear combination of the factors. The equation above illustrates that the PARAFAC decomposition is a parsimonious representation of the original tensor. The original tensor has $\prod_{d=1}^D N_d$ elements and its size grows exponentially with the number of dimensions. On the other hand, the number of degrees of freedom under the PARAFAC representation is $\left( \sum_{d=1}^D N_d\right) K$, and more importantly, it grows additively with the number of dimensions. Notably, the only (hyper-)parameter that needs tuning in the PARAFAC model defined in \eqref{eq::parafac_model} is the rank $K$, which will emerge as one of the advantages of the policy models we will propose in the following section.

Finally, the same way matrices can be reshaped into vectors, tensors can be reshaped into smaller tensors, with the most common reshaping being writing a tensor as a very tall matrix whose columns correspond to one of the original dimensions. Interestingly, with $\odot$ denoting the Khatri-Rao (column-wise Kronecker) product, under the PARAFAC decomposition, the matricization along the $d$th mode is
\begin{equation}\mathrm{mat}_d(\tenbX)=(\bbX_1 \odot \hdots \odot \bbX_{d-1} \odot \bbX_{d+1} \odot \hdots \odot \bbX_D) \bbX_d^\Tr.
\end{equation}

\section{Tensor Low-Rank for PG Methods}

This section contains the main contribution of the paper. It starts by introducing \emph{tensor low-rank tensor} policy models for general classes of policies, and specifically for Gaussian and softmax policies. We then use the tensor low-rank policies to design \emph{tensor low-rank} PG, AC and TR algorithms to solve policy-based RL problems. Furthermore, we show that under suitable conditions, the PG algorithm converges to a stationary point of \eqref{eq::pg_cost}.

\subsection{Tensor Policy Models}

When designing parametric functional estimators, it is necessary to trade-off expressiveness and number of parameters. This is essential as we seek a meaningful approximation while taking care of the convergence time and the required sample size. To address this challenge in RL policy design, we propose using  multi-linear policy models that leverage the fact that many MDPs admit lower-dimensional representations \cite{agarwal2020flambe, yang2019harnessing, rozada2024tensor}.
As introduced in the previous section, the design of policy-based RL algorithms with alternative policy parametrizations boils down to postulating meaningful models, and characterizing the policy scores $\nabla_\theta \text{log}\;\pi_\theta$ to tune the parameters of the model. Hence, in this section we introduce the tensor low-rank policy model and derive the associated policy scores. When the policy scores are easy to evaluate, the associated policy model can be readily deployed in a wide variety of policy-based algorithms. The policy scores depend on i) the probability distribution, and ii) the functional form of the parametrization $\theta$. 
To illustrate this, consider the example of a one-dimensional continuous-action space, where actions are modeled as samples from a univariate Gaussian policy. As introduced earlier, the mean and standard deviation are typically parameterized by the mappings $\mu_{\theta_{\mu}}$ and $\sigma_{\theta_{\sigma}}$. NNs are commonly employed to implement these mappings, resulting policy gradients $\nabla_\theta \log \pi_\theta$ with closed-form that depend on the Gaussian probability density function and the NN architecture. While this approach yields tractable expressions that are straightforward to compute, it presents several challenges: (i) convergence issues, (ii) susceptibility to overfitting and overparameterization, and (iii) a failure to exploit the fact that many MDPs can be effectively represented in lower-dimensional spaces \cite{agarwal2020flambe, yang2019harnessing, rozada2024tensor, rozada2024tensorB}.
The first contribution of this paper puts forth a parsimonious multi-linear policy model that involves rearranging the set $\theta$ into a multi-dimensional tensor to then leverage \emph{low rank} via the PARAFAC decomposition.

Let $\ccalS=\ccalS_1 \times \hdots \times \ccalS_D$ denote the $D$-dimensional state space. The state space is assumed to be discrete and the cardinality of $\ccalS_d$ is denoted as $N_d$. Without loss of generality, consider a generic probabilistic policy with parameters $\{\theta_s \}_{s \in \ccalS}$ where, for simplicity, $\theta_s$ denotes the (one-dimensional) parameter of a probability distribution. We propose rearranging the parameters as a tensor $\tenbTheta \in \reals^{N_1 \times \cdots \times N_D}$. Let $\bbi_s = [i_1^s, \hdots, i_D^s]$ denote the index of the tensor associated with the state $s$, then the parameter $\theta_s$ associated with the state $s$ can be retrieved from the tensor as $\theta_s = [\tenbTheta]_{\bbi_s}$. However, the mere rearrangement of the parameters into a tensor format does not portray any benefits per se, as the 
number of parameters of the tensor grows exponentially with $D$.
In order to alleviate this problem, we postulate that the tensor $\tenbTheta$ is low rank with rank $K$. 
Furthermore, we leverage the PARAFAC decomposition in \eqref{eq::parafac_model} that parameterize the tensor through the collection of $D$ factors of rank $K$, i.e., $\Theta = \{ \bbTheta_d \}_{d=1}^D$, where $\bbTheta_d \in \reals^{N_d \times K}$.
This yields the low-rank policy optimization problem:
\begin{equation} \label{eq::tensor_parameters}
\max_{ \Theta } \bar{G}_V( \tenbTheta )~~\text{s.t.}~~[\tenbTheta]_{\bbi_s} = \sum_{k=1}^K \prod_{d=1}^D [\bbTheta_d]_{i^s_d, k}.
\end{equation}
Under the tensor low-rank framework, actions are sampled from probability distributions whose parameters are represented as a multi-linear model directly given by the factors $\Theta$. Thus, the optimization problem \eqref{eq::tensor_parameters} seeks to determine the factors that maximize the cumulative reward. Note that the original tensor $\tenbTheta$ has $\prod_{i=1}^D N_i$ parameters, and grows exponentially with the number of dimensions $D$. 
In contrast, the low-rank model parameterized by the factors $\Theta$ has $\sum_{i=1}^D N_i K$ parameters, growing only linearly with $D$.
As a remark, one can trivially extend the tensor of parameters to accommodate probability distributions with multiple parameters by considering an additional dimension in the tensor, as we will show later. 

We can now define the policy scores of the tensor policy. The policy scores are the vectorized partial derivatives of the log-policy w.r.t. all the entries of the factors $\Theta$

\begin{equation}\label{eq::score_vector_collecting_partial}
    \nabla_{\Theta} \text{log} \pi_\Theta (a | s) = \begin{bmatrix}
        \frac{\partial{\text{log} \pi_\Theta (a | s)}}{\partial{[\bbTheta_1]_{1, 1}}} \\
        \frac{\partial{\text{log} \pi_\Theta (a | s)}}{\partial{[\bbTheta_1]_{1, 2}}} \\
        \vdots \\
        \frac{\partial{\text{log} \pi_\Theta (a | s)}}{\partial{[\bbTheta_D]_{N_D, K}}} \\
    \end{bmatrix}.
\end{equation}

\noindent Then, the partial derivative of $\text{log} \pi_\Theta (a | s)$ w.r.t. the $(i, k)$th entry of $\bbTheta_d$ can be factorized using the chain rule as
\begin{align}
    \label{eq::tensor_policy_scores_b}
    \frac{\partial \text{log} \; \pi_\Theta \left( a \mid s \right)}{\partial \left[ \bbTheta_d \right]_{i, k}} 
    &= \sum_{\hat{s} \in \ccalS} \frac{\partial \text{log} \; \pi_\Theta \left( a \mid s \right)}{\partial \left[ \tenbTheta \right]_{\bbi_{\hat{s}}}} 
    \frac{\partial \left[ \tenbTheta \right]_{\bbi_{\hat{s}}}}{\partial \left[ \bbTheta_d \right]_{i, k}} \\
    \nonumber
    &= \sum_{\hat{s} \in \ccalS} \frac{\partial \text{log} \; \pi_\Theta \left( a \mid s \right)}{\partial \left[ \tenbTheta \right]_{\bbi_{\hat{s}}}} 
    \ccalI_{\bbi_{\hat{s}} = \bbi_s} 
    \frac{\partial \left[ \tenbTheta \right]_{\bbi_{\hat{s}}}}{\partial \left[ \bbTheta_d \right]_{i, k}} \\
    \nonumber
    &= \frac{\partial \text{log} \; \pi_\Theta \left( a \mid s \right)}{\partial \left[ \tenbTheta \right]_{\bbi_s}} 
    \frac{\partial \left[ \tenbTheta \right]_{\bbi_s}}{\partial \left[ \bbTheta_d \right]_{i, k}} \\
    \nonumber
    & = \frac{\partial \text{log} \; \pi_\Theta \left( a \mid s \right)}{\partial \left[ \tenbTheta \right]_{\bbi_s}}  
    \left( \prod_{\substack{j = 1 \\ j \neq d}}^D \left[ \bbTheta_j \right]_{i^{s}_j, k} \right) 
    \ccalI_{i = i^{s}_d}.
\end{align}
where $\ccalI$ is the indicator function, which takes the value $1$ when the condition of the subscript is met and $0$ otherwise. Hence, the expression in \eqref{eq::tensor_policy_scores_b} states that the derivative is only nonzero if $i=i^{s}_d$ and, as a result, only $KD$ of the entries in \eqref{eq::score_vector_collecting_partial} are nonzero. 
Interestingly, the final expression in \eqref{eq::tensor_policy_scores_b} factorizes the entry-wise policy score into two terms. The former depends on the specific probability distribution, while the latter is determined by the (tensor) model. Next, we derive the specific form of \eqref{eq::tensor_policy_scores_b} for two widely-adopted probability distributions in policy-based RL: Gaussian policies for continuous action spaces, and softmax policies for discrete action spaces.

\noindent \textbf{Gaussian policies.} As introduced previously, in setups with a continuous one-dimensional action space $\ccalA$, it is common to model actions as samples of the univariate Gaussian distribution $A_t \sim \ccalN(\mu_{S_t}, \sigma_{S_t})$. Practical implementations often consider $\sigma(S_t)$ to be constant and independent of the state to encourage exploration. Then, we focus on the estimation of $\mu_{S_t}$ and propose rearranging the set of means $\theta_\mu = \{ \mu_{s}\}_{s\in \ccalS}$ into the tensor $\tenbTheta_\mu \in \reals^{N_1 \times \hdots N_D}$, so that $\mu_{s}=[\tenbTheta_\mu]_{\bbi_s}$. We then postulate that $\tenbTheta_\mu$ has rank $K$, with factors $\Theta_\mu = \{ \bbTheta_d^\mu \}_{d=1}^D$. As a result, the entry-wise policy scores defined in \eqref{eq::tensor_policy_scores_b} for the Gaussian case are
\begin{equation}
    \label{eq::gaussian_policy_score}
    \frac{\partial{\text{log} \; \pi_{\Theta_\mu} (a | s)}}{\partial{[\bbTheta_d^\mu]_{i, k}}} = \frac{a - [\tenbTheta_\mu]_{\bbi_s}}{\sigma^2} \left( \prod_{j = 1 \neq d}^D [\bbTheta_j^\mu]_{i^s_d, k}  \right) \ccalI_{i = i^s_d}. 
\end{equation}
\textbf{Softmax policies.} Consider now a one-dimensional discrete-action space $\ccalA=\{ 1, 2, \hdots, C\}$. Actions can be modeled as samples of the probability distribution defined by the softmax function as $A_t \sim \ccalZ(\bbz_{S_t}, \beta)$, where $\bbz_{S_t}$ are a set of logits associated with the state $S_t$, and $\beta$ is a temperature parameter to promote exploratory actions in the initial phases of learning. The softmax function $\ccalZ: \reals^C \mapsto [0,1]^C$ is a vector function that maps $\bbz=[z_1, \hdots, z_C]^\Tr$ into a probability distribution, with its i-th entry being defined as  $\ccalZ(\bbz, \beta)_i = \frac{e^{\beta z_i}}{\sum_{j=1}^C e^{\beta z_j}}$. We can see the logits $\bbz_{s}$ as a fiber of a tensor $\tenbTheta_\bbz \in \reals^{N_1 \times \hdots N_D \times C}$. In this case, $\tenbTheta_\bbz$ has an extra dimension which is indexed with the actions. Then, with a slightly abuse of notation we define  $\bbi_s = [i^s_1, \hdots, i^s_D]$ and $\bbi_{s,a} = [i^{s,a}_1, \hdots, i^{s,a}_{D+1}]=[\bbi_s ; a]$ to index the fiber $\bbz_s = [\tenbTheta]_{\bbi_s}$ and the entry $z_{s, a} = [\tenbTheta_\bbz]_{\bbi_{s, a}}$ respectively. As before, we postulate that $\tenbTheta_\bbz$ is a $K$-rank tensor, and under the PARAFAC model the factors are $\Theta_\bbz = \{ \bbTheta_d^\bbz \}_{d=1}^{D + 1}$. In the softmax case the entry-wise policy-scores in \eqref{eq::tensor_policy_scores_b} are somewhat more complicated, since we have to take derivatives w.r.t. the fiber of $\tenbTheta_\bbz$ corresponding to the index $\bbi_s$, leading to
\begin{align}
    \label{eq::softmax_policy_score}
    \nonumber
    &\frac{\partial{\text{log} \; \pi_{\Theta_\bbz} (a | s)}}{\partial{[\bbTheta_d^\bbz]_{i, k}}} =  \\
    &\sum_{b = 1}^C  \left( \ccalI_{a = b} - \ccalZ([\tenbTheta_\bbz]_{\bbi_s}, \beta)_b\right) \! \left(\prod_{j = 1 \neq d}^{D+1} [\bbTheta_j^\bbz]_{i^{s, b}_j, k} \right) \! \ccalI_{i=i^{s, b}_d}, 
\end{align}
where, again, the partial derivatives of $\text{log} \pi_{\Theta_\bbz} (a | s)$ w.r.t. other fibers $\bbi \neq \bbi_s$ are zero.

\noindent \textbf{Sampling the MDP.} As explained in the previous section, since the dynamics of the MDP are unknown, RL learns by interacting with the environment. A common mechanism underlying various algorithms discussed subsequently involves sampling trajectories from the MDP. Specifically, we sample a set of trajectories $\{\tau_u\}_{u=1}^U \sim g(\Theta_\mu)$ following the Gaussian policy $\pi_{\Theta_\mu}$ for continuous actions or $\{\tau_u\}_{u=1}^U \sim g(\Theta_\bbz)$ following the softmax policy $\pi_{\Theta_\bbz}$ for discrete actions. The sampling process for Gaussian policies is outlined in Algorithm \ref{alg::gsampling}, while the sampling mechanism for softmax policies is depicted in Algorithm \ref{alg::ssampling}.

\begin{algorithm}[!htbp]
\flushleft
\caption{Tensor Low-Rank Gaussian-Policy Sampling}
\label{alg::gsampling}
\begin{algorithmic}
    \Require Policy and factors $\Theta_\mu$; fixed standard deviation $\sigma$; number of trajectories per sample $U$; and number of time-steps per trajectory $T$.
    \State
    \Function{GaussianSampling}{$\Theta_\mu$, $U$, $T$, $\sigma$}
        \For{$u=1, \hdots, U$}
            \State{Observe $S_1^u$}
            \For{$t=1, ..., T$}
                \State{$\mu_{S_t^u} \gets \sum_{k=1}^K \prod_{d=1}^D [\bbTheta_d^\mu]_{i^{S_t^u}_d, k}$}
                \State{$A_t^u \sim \ccalN(\mu_{S_t^u}, \sigma)$}
                \State{Take $A_t^u$, observe $S_{t+1}^u$, and $R_t^u$}
                \State{$S_t^u \gets S_{t+1}^u$}
            \EndFor
            \State $\tau_u = \{(S_t^u, A_t^u)\}_{t=1}^T$
        \EndFor
        \State \Return{$\{\tau_u\}_{u=1}^U$}
        \State
    \EndFunction
\end{algorithmic}
\end{algorithm}

\begin{algorithm}[!htbp]
\flushleft
\caption{Tensor Low-Rank Softmax-Policy Sampling}
\label{alg::ssampling}
\begin{algorithmic}
    \Require Policy and factors $\Theta_\bbz$; temperature coefficient $\beta$; number of trajectories per sample $U$; and number of time-steps per trajectory $T$.
    \State
    \Function{SoftmaxSampling}{$\Theta_\bbz$, $U$, $T$, $\beta$}
        \For{$u=1, \hdots, U$}
            \State{Observe $S_1^u$}
            \For{$t=1, ..., T$}
                \State{$[\bbz_{S_t^u}]_a \gets \sum_{k=1}^K \prod_{d=1}^{D + 1} [\bbTheta_d^\bbz]_{i^{S_t^u, A_t^u}_d, k} \;  \forall a \in \ccalA $}
                \State{$A_t^u \sim \ccalZ(\bbz_{S_t^u}, \beta)$}
                \State{Take $A_t^u$, observe $S_{t+1}^u$, and $R_t^u$}
                \State{$S_t^u \gets S_{t+1}^u$}
            \EndFor
            \State $\tau_u = \{(S_t^u, A_t^u)\}_{t=1}^T$
        \EndFor
        \State \Return{$\{\tau_u\}_{u=1}^U$}
        \State
    \EndFunction
\end{algorithmic}
\end{algorithm}

\noindent \textbf{Multi-dimensional action spaces.} The framework presented in this paper can accommodate action spaces with multiple dimensions, denoted as $\ccalA = \ccalA_1 \times \hdots \times \ccalA_P$. Actions within this framework are represented as $P$-dimensional random vectors $a=[a_1, \hdots, a_P]^\Tr$, sampled from a multivariate policy distribution. Traditionally, RL algorithms assume independence among action dimensions. For Gaussian policies this leads to a multivariate Gaussian distribution with an identity covariance matrix, which in the tensor framework can be modeled by introducing a new dimension in the tensor of means $\tenbTheta_\mu$, where $\bbmu_s = [\tenbTheta_\mu]_{\bbi_s}$ represents a vector of means. Similarly, for policies based on the softmax distribution, introducing a set of logits per action dimension necessitates a new dimension in the tensor of logits $\tenbTheta_\bbZ$, where $\bbZ_s = [\tenbTheta_\bbZ]_{\bbi_s}$ represents a matrix of logits. Due to the assumption of independent actions, the joint policy can be factorized as $\pi_\Theta (a | s) = \prod_{p=1}^P \pi_\Theta (a_p | s)$, simplifying the policy scores computation to $\nabla_\Theta \log \pi_\Theta (a | s) = \sum_{p=1}^P \nabla_\Theta \log \pi_\Theta (a_p | s)$. Hence, it is straightforward to extend the derivation of the policy scores presented in this section to the multi-dimensional case. For the sake of clarity, subsequent derivations in the paper focus on the unidimensional action space.

\subsection{Tensor Low-Rank Actor-Critic Methods}
\label{sec::TLRAC}

Tensor low-rank policies can be readily integrated into PG and AC methods, since both are first-order methods. In this section, we introduce a \emph{tensor low-rank} PG algorithm, and a \emph{tensor low-rank} AC algorithm.

\noindent \textbf{Policy Gradients.} Recall that in PG we aim to maximize  the function $\bar{G}_V$ defined in \eqref{eq::pg_cost} via stochastic gradient ascent. Again, we use the Gaussian policy $\pi_{\Theta_\mu}$ for continuous action spaces, and the softmax policy $\pi_{\Theta_\bbz}$ for discrete action spaces. For the Gaussian case, the first step involves sampling a set of trajectories from the MDP $\{\tau_u\}_{u=1}^U \sim g(\Theta_\mu)$ as indicated in Algorithm \ref{alg::gsampling}. Then, we compute the sample returns $G_t^u = \sum_{t'=t}^T R_{t'}^u$, and the policy scores $\nabla_{\Theta_\mu} \text{log} \pi_{\Theta_\mu} (A_t^u | S_t^u)$ using \eqref{eq::gaussian_policy_score}. Finally, the factors in $\Theta_\mu$ are updated leveraging the stochastic approximation of the gradient defined in \eqref{eq::pg_gradient}
\begin{equation}
    \label{eq::gaussian_stochastic_gradient}
    \tilde{\nabla}_{\Theta_\mu} \bar{G}_V(\Theta_\mu) = \frac{1}{U} \sum_{u=1}^U \sum_{t=1}^T G_t^u \nabla_{\Theta_\mu} \text{log} \pi_{\Theta_\mu} (A_t^u | S_t^u).
\end{equation}

\noindent Analogously, in the softmax case  we sample from the MDP according to $\pi_{\bbz}$ following Algorithm \ref{alg::ssampling}, we compute the policy scores $\nabla_{\Theta_\bbz} \text{log} \pi_{\Theta_\bbz} (A_t^u | S_t^u)$ as per \eqref{eq::softmax_policy_score}, and we obtain the stochastic gradient
\begin{equation}
    \label{eq::softmax_stochastic_gradient}
    \tilde{\nabla}_{\Theta_\bbz} \bar{G}_V(\Theta_\bbz) = \frac{1}{U} \sum_{u=1}^U \sum_{t=1}^T G_t^u \nabla_{\Theta_\bbz} \text{log} \pi_{\Theta_\bbz} (A_t^u | S_t^u).
\end{equation}

The result is a tensor low-rank policy gradient (TLRPG) algorithm that is described in Algorithm \ref{alg::pg}. The Gaussian case involves using the factors $\Theta=\Theta_\mu$,  the learning rate $\eta$, and using the Gaussian sampling defined in Algorithm \ref{alg::gsampling} with $b=0$ and $\text{params}_b=\sigma$. Similarly, in the softmax case $\Theta=\Theta_\bbz$, and we use the softmax sampling defined in Algorithm \ref{alg::ssampling} with $b=1$, and $\text{params}_b=\beta$.

\begin{algorithm}[!htbp]
\flushleft
\caption{Tensor Low-Rank Policy Gradient (TLRPG)}
\label{alg::pg}
\begin{algorithmic}[1]
    \Require Initial policy and factors $\Theta^1$; variable $b$ that triggers Algorithm \ref{alg::gsampling} if $b=0$, and Algorithm \ref{alg::ssampling} if $b=1$; the parameters of the sampling function $\text{params}_b$; learning rates $\{\eta^{h} \}_{h=1}^H$; number of trajectories per sample $U$; maximum number of iterations $H$; and number of time-steps per trajectory $T$.
    \For{$h=1, \hdots,  H$}
        \State $\{\tau_u^h\}_{u=1}^U \gets \textsc{Sampling}(\Theta^{h}, U, T, b, \text{params}_b)$ 
        \State $\Theta^{h+1} \gets \Theta^h + \eta^h \tilde{\nabla}_{\Theta} \bar{G}_V (\Theta^h)$ \label{alg:sgd_line}
    \EndFor
\end{algorithmic}
\end{algorithm}

\noindent \textbf{Actor Critic.} Tensor low-rank policies can be generalized to AC setups. The first steps are quite similar. Namely, following Algorithms \ref{alg::gsampling} and \ref{alg::ssampling}, the actor $\pi_{\Theta_\mu}$ or $\pi_{\Theta_\bbz}$ is used to obtain a set of trajectories from the MDP $\{\tau_u\}_{u=1}^U$, and the samples are used to compute the policy scores according to \eqref{eq::gaussian_policy_score} or \eqref{eq::softmax_policy_score}. However, AC methods subtract the approximated VF from the sample return to estimate the advantage function $\mathbb{A}_\omega(S_t^u) = G_t^u - \mathbb{V}_\omega(S_t^u)$, resulting in algorithms with lower variance. The price to pay is that we need to learn the parameters $\omega$ of the VF approximator, or critic, minimizing the quadratic cost $L (\omega)=\frac{1}{2}\sum_{u=1}^U \sum_{t=1}^T (G_t^u - \mathbb{V}_\omega(S_t^u))^2$. Following the spirit of this paper, we propose a multi-linear model to approximate the VFs as well. More specifically, we collect the VFs in the tensor $\tenbV \in \reals^{\ccalS_1 \times \hdots \times \ccalS_D}$ to then postulate a $K'$-rank PARAFAC model with factors $\omega = \{\bbV_d \in \reals^{N_d \times K'}\}_{d=1}^D$. Then, the VF of $s$ is $\mathbb{V}_{\omega}(s) = [\tenbV]_{\bbi_s}$, which under the PARAFAC model are computed as $\sum_{k=1}^{K'}\prod_{d=1}^D [\bbV_d]_{i^s_d, k}$. In AC methods, the updates of the parameters of the actor and the critic are alternated. Firstly, we fix the critic and update the actor. In the Gaussian case this entails updating the factors $\Theta_\mu$ of the actor using the stochastic gradient
\begin{equation}
    \label{eq::gaussian_stochastic_gradient_ac}
    \tilde{\nabla}_{\Theta_\mu} \bar{G}_A(\Theta_\mu) = \frac{1}{U} \sum_{u=1}^U \sum_{t=1}^T \mathbb{A}_{\omega}(S_t^u) \nabla_{\Theta{\mu}} \text{log} \pi_{\Theta_\mu} (A_t^u | S_t^u),
\end{equation}
where $\mathbb{A}_{\omega}(S_t^u) = G_j^u - [\tenbV]_{\bbi_{S_t^u}}$. In the softmax case, we update the factors $\Theta_\bbz$ of the actor using the stochastic gradient
\begin{equation}
    \label{eq::softmax_stochastic_gradient_ac}
    \tilde{\nabla}_{\Theta_\bbz} \bar{G}_A(\Theta_\bbz) = \frac{1}{U} \sum_{u=1}^U \sum_{t=1}^T \mathbb{A}_{\omega}(S_t^u) \nabla_{\Theta{\bbz}} \text{log} \pi_{\Theta_\bbz} (A_t^u | S_t^u).
\end{equation}

\noindent Secondly, we update the critic by minimizing $L(\omega)$ via gradient descent. In order to do so we need to compute the  gradient $\nabla_{\omega} L (\omega)$, which entry-wise is
\begin{equation}
    \frac{\partial{L (\omega)}}{\partial{[\bbV_d]_{i, k}}}\! = \!\!\sum_{u=1}^U \sum_{t=1}^T \!\left( G_t^u \!-\! [\tenbV]_{\bbi_{S_t^u}}\right) \!\left( \prod_{j = 1 \neq d}^D\!\! [\bbV_j]_{i^{S_t^u}_j, k} \!\right) \ccalI_{i = i^{S_t^u}_d}.
\end{equation}
\noindent This leads to a tensor low-rank actor-critic (TLRAC) algorithm which is fully depicted in Algorithm \ref{alg::ac}. In the Gaussian case $\Theta=\Theta_\mu$, and we use the sampling mechanism defined in Algorithm \ref{alg::gsampling} with $b=0$, and $\text{params}_b=\sigma$, while in the softmax case we set $\Theta=\Theta_\bbz$ and  the learning rate for the critic $\alpha$, and we use the sampling in Algorithm \ref{alg::ssampling} with $b=1$, and $\text{params}_b=\beta$.

\begin{algorithm}[!htbp]
\flushleft
\caption{Tensor Low-Rank Actor Critic (TLRAC)}
\label{alg::ac}
\begin{algorithmic}
    \Require Initial policy and factors $\Theta^1$; initial VFs and factors $\omega^1$; variable $b$ that triggers Algorithm \ref{alg::gsampling} if $b=0$, and Algorithm \ref{alg::ssampling} if $b=1$; the parameters of the sampling function $\text{params}_b$ ; actor learning rates $\{\eta^{h} \}_{h=1}^H$; critic learning rates $\{\alpha^{h} \}_{h=1}^H$; number of trajectories per sample $U$; maximum number of iterations $H$; and number of time-steps per trajectory $T$.
    \For{$h=1, \hdots,  H$}
        \State $\{\tau_u^h\}_{u=1}^U \gets \textsc{Sampling}(\Theta^{h}, U, T, b, \text{params}_b)$
        \State $\Theta^{h+1} \gets \Theta^h + \eta^h \tilde{\nabla}_{\Theta} \bar{G}_A(\Theta^h)$ \hfill\Comment{Actor update}
        \State $\omega^{h+1} \gets \omega^h + \alpha^h \nabla_{\omega} L(\omega^h)$ \hfill\Comment{Critic update}
    \EndFor
\end{algorithmic}
\end{algorithm}

\subsection{Tensor Low-Rank Trust-Region Methods}

We claimed in previous sections that estimating the parameters of a policy via TR methods can be reduced to characterizing the policy scores too. In this section, we show first how TRPO and PPO can be expressed in terms of the policy scores, to then leverage our \emph{tensor low-rank} policy models in both setups.

\noindent \textbf{TRPO.} We begin by recalling that in TRPO we aim to maximize $\bar{G}_T$ as defined in \eqref{eq::trpo_cost}, subject to the constraint that the KL-divergence between $\pi_\Theta$ and the (old) sampling policy $\pi_{\tilde{\Theta}}$ is less than $\delta$. More formally, we solve the optimization
\begin{equation}
\label{eq::trpo}
\begin{aligned}
& \underset{\Theta}{\text{maximize}} \;\; \bar{G}_{T}(\Theta) = && \mathbb{E}_{ g (\tilde{\Theta}) } \left[ \sum_{t=1}^T q_{\Theta, \tilde{\Theta}} (S_t, A_t) \mathbb{A}_\omega (S_t) \right] \\
& \text{subject to} && \mathbb{E}\left[D_{KL}(\pi_{\tilde{\Theta}}(\cdot|s) || \pi_\Theta(\cdot|s))\right] \le \delta
\end{aligned}
\end{equation}
\noindent via sampling-based techniques. To render \eqref{eq::trpo} tractable, we resort to Taylor expansions of both the objective function and the KL constraint \cite{kakade2001natural, schulman2015trust}. Leveraging a first-order approximation for the objective $\bar{G}_T$, and a second-order approximation for the KL constraint, the following quadratic optimization is reached:

\begin{equation}
\label{eq::trpo_relaxed}
\begin{aligned}
    & \underset{\Theta}{\text{maximize}} && \tilde{\bbg}^\Tr(\Theta - \tilde{\Theta}) \\
    & \text{subject to} && (\Theta - \tilde{\Theta})^\Tr\tilde{\bbH} (\Theta - \tilde{\Theta}) \le \delta,
\end{aligned}
\end{equation}

\noindent where $\tilde{\bbg} = \nabla_\Theta \bar{G}_T(\Theta) \left. \right|_{\Theta = \tilde{\Theta}}$ is the gradient of the cost function w.r.t. the set of parameters $\Theta$ evaluated at $\tilde{\Theta}$, and $\tilde{\bbH}$ is the Hessian matrix of the KL constraint evaluated at $\tilde{\Theta}$, with entries $[\tilde{\bbH}]_{i, j} = \frac{\partial}{\partial {\Theta}_i} \frac{\partial} {\partial {\Theta}_j} \mathbb{E} \left[ D_{KL}(\pi_{\tilde{\Theta}}(\cdot|s) || \pi_\Theta(\cdot|s)) \right]  \left. \right|_{\Theta = \tilde{\Theta}}$. Notably, the gradient of the KL constraint evaluated at $\Theta=\tilde{\Theta}$ is zero and the Hessian evaluated at $\Theta=\tilde{\Theta}$ is the Fisher information matrix (FIM) \cite{amari2012differential}. The FIM definition allows us to reformulate $\tilde{\bbH}$ in terms of the outer product of the policy scores as $\tilde{\bbH} = \mathbb{E} \left[ \nabla_{\Theta} \text{log} \; \pi_{{\Theta}}(a|s) \nabla_{\Theta} \text{log} \; \pi_{{\Theta}}(a|s)^\Tr \right] \left. \right|_{\Theta = \tilde{\Theta}}$. Trivially, the gradient  $\tilde{\bbg}$ can also be reformulated in terms of the policy scores $\nabla_{\Theta} \text{log} \; \pi_\Theta$ as $\tilde{\bbg} = \mathbb{E}_{g(\tilde{\Theta})} \left[ \sum_{t=1}^T q_{\Theta, \tilde{\Theta}} (S_t, A_t) \nabla_{\Theta} \text{log} \; \pi_{{\Theta}}(A_t|S_t) \mathbb{A}_\omega (S_t) \right]\Big|_{\Theta = \tilde{\Theta}}$. Recall that in TRPO we approximate the expectations in $\tilde{\bbg}$ and $\tilde{\bbH}$  using sampled trajectories from the MDP $\{\tau_u\}_{u=1}^U \sim g(\tilde{\Theta})$ following $\pi_{\tilde{\Theta}}$. Thus we end up having the following stochastic approximation
\begin{align}
    \label{eq::trpo_stochastic_approximation_g}
    \tilde{\bbg} \! &\approx \! \frac{1}{U} \! \sum_{u=1}^U \sum_{t=1}^T \! q_{\Theta, \tilde{\Theta}} (S_t^u, A_t^u) \nabla_{\Theta} \text{log} \; \pi_{{\Theta}}(A_t^u|S_t^u) \mathbb{A}_\omega (S_t^u) \Bigg|_{\Theta = \tilde{\Theta}} \! , \\
    \label{eq::trpo_stochastic_approximation_h}
    \tilde{\bbH} \! &\approx \! \frac{1}{U} \frac{1}{T} \! \sum_{u=1}^U \sum_{t=1}^T  \! \nabla_{\Theta} \text{log} \;\pi_{{\Theta}}(A_t^u|S_t^u) \! \nabla_{\Theta} \text{log} \;\pi_{{\Theta}}(A_t^u|S_t^u)^\Tr \Big|_{\Theta = \tilde{\Theta}} \! .
\end{align}
As previously pointed out, equations \eqref{eq::trpo_stochastic_approximation_g} and \eqref{eq::trpo_stochastic_approximation_h} show that the policy model  impacts the computation of $\tilde{\bbg}$ and $\tilde{\bbH}$ solely through the policy scores.

Now, everything is set up to introduce our \emph{tensor low-rank} policy model in TRPO. First, we introduce the initial Gaussian policy $\pi_{\tilde{\Theta}_{\mu}}$ for continuous action setups, and the initial softmax policy $\pi_{\tilde{\Theta}_{\bbz}}$ for discrete action spaces. Then, the actor step involves sampling a set of trajectories from the MDP $\{\tau_u\}_{u=1}^U$ following $\pi_{\tilde{\Theta}_{\mu}}$ or $\pi_{\tilde{\Theta}_{\bbz}}$ as per Algorithms \ref{alg::gsampling} and \ref{alg::ssampling}. We compute the sample returns $G_t^u = \sum_{t'=t}^T R_{t'}^u$, and the VFs using the critic $\mathbb{V}_\omega(S_t^u)$. Following \eqref{eq::trpo_stochastic_approximation_g} and \eqref{eq::trpo_stochastic_approximation_h}, Gaussian policies lead to the stochastic gradient $\tilde{\bbg}_\mu$ and the stochastic Hessian $\tilde{\bbH}_\mu$ using the Gaussian policy scores in \eqref{eq::gaussian_policy_score}, and similarly, softmax policies lead to $\tilde{\bbg}_\bbz$ and $\tilde{\bbH}_\bbz$ using the softmax policy scores in \eqref{eq::softmax_policy_score}.  Lastly, we solve \eqref{eq::trpo_relaxed} via conjugate gradient descent methods \cite{schulman2015trust} to obtain $\Theta_\mu$ in the Gaussian case and $\Theta_\bbz$ in the softmax case. After completing the actor step, the critic is updated in the exact same way as in Sec. \ref{sec::TLRAC}. The resulting trust-region tensor low-rank policy optimization (TRTLRPO) algorithm is depicted in Algorithm \ref{alg::trpo}. Again, the Gaussian case involves setting $\Theta=\Theta_\mu$, and sampling according to Algorithm \ref{alg::gsampling} with $b=0$ and $\text{params}_b=\sigma$. The softmax case involves setting $\Theta=\Theta_\bbz$, and sampling according to Algorithm \ref{alg::ssampling} with $b=1$ and $\text{params}_b=\beta$.

\begin{algorithm}[!htbp]
\flushleft
\caption{Trust-Region Tensor Low-Rank Policy Optimization (TRTLRPO)}
\label{alg::trpo}
\begin{algorithmic}
    \Require Initial policy and factors $\Theta^1$; initial VFs and factors $\omega^1$; variable $b$ that triggers Algorithm \ref{alg::gsampling} if $b=0$, and Algorithm \ref{alg::ssampling} if $b=1$; the parameters of the sampling function $\text{params}_b$ ; critic learning rates $\{\alpha^{h} \}_{h=1}^H$; number of trajectories per sample $U$; maximum number of iterations $H$; and number of time-steps per trajectory $T$.
    \For{$h=1, \hdots,  H$}
        \State $\{\tau_u^h\}_{u=1}^U \gets \textsc{Sampling}(\Theta^{h}, U, T, b, \text{params}_b)$
        \State $\Theta^{h+1} \gets$ Solve \eqref{eq::trpo_relaxed} using $\tilde{\bbg}$ and $\tilde{\bbH}$ \hfill\Comment{Actor update}
        \State $\omega^{h+1} \gets \omega^h + \alpha^h \nabla_{\omega} L(\omega^h)$ \hfill\Comment{Critic update}
    \EndFor
\end{algorithmic}
\end{algorithm}

\noindent \textbf{PPO.} In PPO we seek to maximize the unconstrainted function $\bar{G}_P$ defined in \eqref{eq::ppo_cost}, which enforces a trust-region via the clipped probability ratio $q_{\Theta, \tilde{\Theta}}$. To solve this problem, we resort to first order methods, hence we need to compute the gradient $\nabla_\Theta \bar{G}_P(\Theta)$. Observing the definition of $q_{\Theta, \tilde{\Theta}}$ in \eqref{eq::ratio_ppo}, the term inside the expectation of  \eqref{eq::ppo_cost} can be rearranged as
\begin{align}
    \label{eq::ppo_inside_term}
    &\min \{ q_{\Theta, \tilde{\Theta}} (S_t, A_t) \mathbb{A}_\omega (S_t), \; \hat{q}_{\Theta, \tilde{\Theta}} (S_t, A_t) \mathbb{A}_\omega (S_t)\} = \\
    \nonumber
            &\begin{cases}
              (1 - \epsilon)  \mathbb{A}_\omega (S_t) \!\! & \text{if} \; q_{\Theta, \!\tilde{\Theta}} (S_t, A_t) \! \leq \!1 \!-\! \epsilon \; \text{and} \;\mathbb{A}_\omega (S_t) \!\le \!0 \\
              (1 + \epsilon)  \mathbb{A}_\omega (S_t) \!\!  & \text{if} \; q_{\Theta,\! \tilde{\Theta}} (S_t, A_t) \! \geq \!1 \!+\! \epsilon \; \text{and} \;\mathbb{A}_\omega (S_t) \!\ge\! 0\\
              q_{\Theta, \tilde{\Theta}} (S_t, A_t) \mathbb{A}_\omega (S_t) \!\! & \text{otherwise}.
            \end{cases}
\end{align}

\noindent Then, it is easy to see that, conditioned to the value of $\mathbb{A}_\omega (S_t)$, the gradient of $\bar{G}_P$ w.r.t. $\Theta$ outside the clipping range defined by $\epsilon$ is $0$. Let $\ccalI_{\Theta, \tilde{\Theta}}^\omega (S_t, A_t)$ be an indicator function that takes the value $1$ when the third clause of \eqref{eq::ppo_inside_term} is active, and $0$ otherwise. The gradient of \eqref{eq::ppo_cost} is defined as
\begin{align}
    \nonumber
    &\nabla_\Theta \bar{G}_P(\Theta) = \mathbb{E}_{g \left( \tilde{\Theta} \right)} \Big[ \sum_{t=1}^T \ccalI_{\Theta, \tilde{\Theta}}^\omega (S_t, A_t)  q_{\Theta, \tilde{\Theta}}(S_t, A_t) \\
    \label{eq::ppo_gradient}
    &\hspace{3.2cm}\nabla_\Theta \text{log} \pi_\Theta(A_t | S_t) \mathbb{A}_\omega (S_t) \Big],
\end{align}
\noindent where again, computing the gradient of $\bar{G}_P$ boils down to estimating the policy scores $\nabla_\Theta \text{log} \pi_\Theta(A_t | S_t)$.

Leveraging \emph{tensor low-rank} policy models in PPO involves using the Gaussian policy $\pi_{\tilde{\Theta}_{\mu}}$ or the softmax policy $\pi_{\tilde{\Theta}_{\bbz}}$ to first sample trajectories $\{\tau_u\}_{u=1}^U$ from the MDP as depicted in Algorithms \ref{alg::gsampling} and \ref{alg::ssampling}. Next, we compute $G_t^u$, $\mathbb{V}_\omega(S_t^u)$, and the policy scores as per $\eqref{eq::gaussian_policy_score}$ or \eqref{eq::softmax_policy_score}. Then, we update the actor using the stochastic gradients for the Gaussian case
\begin{align}
       \nonumber
    &\nabla_{\Theta_\mu} \bar{G}_P(\Theta_\mu) = \mathbb{E}_{g \left( \tilde{\Theta}_{\mu} \right)} \Big[ \sum_{t=1}^T \ccalI_{\Theta, \tilde{\Theta}}^\omega (S_t, A_t)  q_{\Theta_\mu, \tilde{\Theta}_{\mu}}(S_t, A_t)\\
 \label{eq::ppo_gradient_gaussian}
    & \hspace{3.2cm}\nabla_{\Theta_\mu} \text{log} \pi_{\Theta_\mu}(A_t | S_t) \mathbb{A}_\omega (S_t) \Big],
\end{align}
\noindent and for the softmax case
\begin{align}
\nonumber
    &\nabla_{\Theta_\bbz} \bar{G}_P(\Theta_\bbz) = \mathbb{E}_{g \left( \tilde{\Theta}_{\bbz} \right)} \Big[ \sum_{t=1}^T \ccalI_{\Theta, \tilde{\Theta}}^\omega (S_t, A_t)  q_{\Theta_\bbz, \tilde{\Theta}_{\bbz}}(S_t, A_t)\\
    \label{eq::ppo_gradient_softmax}
    &\hspace{3.2cm} \nabla_{\Theta_\bbz} \text{log} \pi_{\Theta_\bbz}(A_t | S_t) \mathbb{A}_\omega (S_t) \Big].
\end{align}
\noindent Again, the critic is updated after completing the actor step. The resulting algorithm is called proximal tensor low-rank policy optimization (PTLRPO), and is outlined in Algorithm \ref{alg::ppo}. In the Gaussian case we set $\Theta=\Theta_\mu$, and we sample according to Algorithm \ref{alg::gsampling} with $b=0$ and $\text{params}_b=\sigma$. In the softmax case, we set $\Theta=\Theta_\bbz$, and we sample as in Algorithm \ref{alg::ssampling} with $b=1$ and $\text{params}_b=\beta$.

\begin{algorithm}[!htbp]
\flushleft
\caption{Proximal Tensor Low-Rank Policy Optimization (PTLRPO)}
\label{alg::ppo}
\begin{algorithmic}
    \Require Initial policy and factors $\Theta^1$; initial VFs and factors $\omega^1$; variable $b$ that triggers Algorithm \ref{alg::gsampling} if $b=0$, and Algorithm \ref{alg::ssampling} if $b=1$; the parameters of the sampling function $\text{params}_b$ ; actor learning rates $\{\eta^{h} \}_{h=1}^H$; critic learning rates $\{\alpha^{h} \}_{h=1}^H$; number of trajectories per sample $U$; maximum number of iterations $H$; and number of time-steps per trajectory $T$.
    \For{$h=1, \hdots,  H$}
        \State $\{\tau_u^h\}_{u=1}^U \gets \textsc{Sampling}(\Theta^{h}, U, T, b, \text{params}_b)$
        \State $\Theta^{h+1} \gets \Theta^h + \eta^h \tilde{\nabla}_{\Theta} \bar{G}_P(\Theta^h)$ \hfill\Comment{Actor update}
        \State $\omega^{h+1} \gets \omega^h + \alpha^h \nabla_{\omega} L(\omega^h)$ \hfill\Comment{Critic update}
    \EndFor
\end{algorithmic}
\end{algorithm}

\subsection{Convergence Analysis} \label{subsec:analysis}
We focus on the tensor low rank PG (TLRPG) method in Algorithm~\ref{alg::pg} and analyze its convergence towards a stationary solution of \eqref{eq::pg_cost}. To make the analysis tractable, we consider a slight modification to the algorithm which composes of a \emph{projection} step on top of the stochastic policy gradient update. Let $\ccalO \subset  \mathbb{R}^{ N_1 \times K } \times \hdots \times \mathbb{R}^{ N_D \times K }$ be a closed and convex set.
In lieu of Line~\ref{alg:sgd_line} of Algorithm~\ref{alg::pg}, the update for $\Theta$ is
\begin{equation} \label{eq::proj_tlrpg}
    \Theta^{h+1} = {\cal P}_{\ccalO} ( \Theta + \eta^h \tilde{\nabla}_{\Theta} \bar{G}_V (\Theta^h) ),
\end{equation}
where ${\cal P}_{\ccalO}(\cdot)$ denotes the Euclidean projection onto $\ccalO$.
For the sake of clarity, we define the shorthand notation $\nabla_{ \bbTheta_i } \log p_{g(\Theta)} (\tau_T) := \sum_{t=1}^T \nabla_{ \bbTheta_i } \log \pi_{ \Theta } ( A_t | S_t )$ and consider: 
\begin{assumption} \label{assumption:Bd}
The reward function is uniformly bounded, i.e., there exists a constant ${\rm R}$ such that 
\[
|r( s, a, s' )| \leq {\rm R},~\forall~s, s' \in {\cal S}, a \in {\cal A}.
\]
\end{assumption}
\begin{assumption} \label{assumption:GradPolicy}
For any $\Theta, \Theta' \in \ccalO$, $s \in {\cal S}$, $a \in {\cal A}$, $i = 1,...,D$, there exists a constant $L_0$ such that 
\[
\| \nabla_{ \bbTheta_i } \log \pi_\Theta ( a | s ) - \nabla_{ \bbTheta_i } \log \pi_{\Theta'} ( a | s ) \| \leq L_0 \| \Theta - \Theta' \|.
\]
\end{assumption}
\begin{assumption} \label{assumption:Distribution}
For any $\Theta, \Theta', \Theta'' \in \Theta$, $i = 1,...,D$, there exists a constant $L_1$ such that
\[
\begin{split} 
& \Big\| \mathbb{E}_{ \tau_T \sim g(\Theta) } \left[ {G}_T ( \tau_T ) \nabla_{ \bbTheta_i } \log p_{g(\Theta'')} ( \tau_T )  \right] \\
& - \mathbb{E}_{ \tau_T' \sim g (\Theta') } \left[ {G}_T ( \tau_T' ) \nabla_{ \bbTheta_i } \log p_{ g (\Theta'')} ( \tau_T' )  \right] \Big\| \leq L_1 \| \Theta - \Theta' \|.
\end{split}
\]
\end{assumption}
\begin{assumption} \label{assumption:StocGrad}
    For any $\Theta$, there exists a constant ${\tt G}$ such that 
    \[
    \mathbb{E}_{ \tau_T \sim g(\Theta) } \left[ \left\| G_T( \tau_{T} ) \nabla_{ \Theta } \log p_{ g(\Theta )} ( \tau_{T} ) \right\|^2 \right] \leq {\tt G}^2.
    \]
\end{assumption}
Assumption \ref{assumption:Bd} is a standard condition used in the RL literature. 
Meanwhile, Assumptions \ref{assumption:GradPolicy}--\ref{assumption:StocGrad} can be verified for the Gaussian and softmax policies if $\Theta$ is compact; see Appendix~\ref{app:assumption}. To handle the projected stochastic PG algorithm \eqref{eq::proj_tlrpg} for non-convex optimization problems, our analysis relies on the framework in \cite{davis2019stochastic} that utilizes the following Moreau envelope of the objective function in \eqref{eq::pg_cost}: for any $\lambda > 0$, we set
\[
    \bar{G}_V^{ \lambda } ( \Theta ) := \min_{ \vartheta \in \ccalO } \left\{ -\bar{G}_V (\vartheta) + \frac{1}{2 \lambda} \| \Theta - \vartheta \|_F^2 \right\} .
\]
For sufficiently small $\lambda$, it holds for any $\Theta$ \cite{davis2019stochastic},
\begin{equation} \label{eq:moreau_bd}
\left\| \frac{1}{\lambda} ( \Theta - {\cal P}_{\ccalO}( \Theta + \lambda \nabla \bar{G}_V(\Theta) ) \right\|^2 \leq 3 \, \| \nabla \bar{G}_V^{ \lambda } ( \Theta ) \|^2,
\end{equation} 
which leads to the following convergence guarantee for the projected TLRPG method.
\begin{theorem}
    Under Assumptions~\ref{assumption:Bd}--\ref{assumption:StocGrad}. For any $\bar{L} > L_G := D( T^2 L_0 {\rm R} + L_1 )$ and maximum iteration number $H \ge 1$, the following holds for $\{ \Theta^h \}_{h =1}^H$ generated by \eqref{eq::proj_tlrpg}:
    \begin{equation}
    \begin{split}
    & \textstyle (1/H) \sum_{h=1}^H \mathbb{E} \left[ \| \nabla_{\Theta} \bar{G}_V^{ 1/\bar{L} } ( \Theta^{h} ) \|^2 \right] \\
    & \leq \frac{ \bar{L} }{ \bar{L} - L_G } \frac{ \bar{G}_V^{ 1/\bar{L}} (\Theta^1)  - \bar{G}_V^{1/\bar{L}, \star} + \frac{ \bar{L} {\tt G}^2 }{ 2 } \sum_{h=1}^H (\eta^h)^2 }{ \sum_{h=1}^H \eta^h },
    \end{split}
    \end{equation}
    where $\bar{G}_V^{1/\bar{L}, \star} := \min_{ \Theta \in \ccalO} \bar{G}_V^{ 1/\bar{L} } ( \Theta )$.
\end{theorem}
\begin{proof}
The proof follows directly from \cite[Theorem 3.1]{davis2019stochastic}. In particular, it suffices to show that $\nabla \bar{G}_V(\Theta)$ is $L_G$-Lipschitz continuous, which implies that $\bar{G}_V(\Theta)$ is $L_G$-weakly convex. We observe that for any $i = 1,...,D$,
\begin{equation*}
\begin{split}
& \| \nabla_{\bbTheta_i} \bar{G}_V(\Theta) - \nabla_{\bbTheta_i} \bar{G}_V(\Theta') \| \\
& \leq \left\| \mathbb{E}_{ \tau_T \sim g(\Theta) } \left[ G_T ( \tau_T ) ( \nabla_{ \bbTheta_i } \log p_{ g(\Theta)} ( \tau_T ) - \log p_{ g(\Theta')} ( \tau_T ) )  \right] \right\| \\
& + \| \mathbb{E}_{ \tau_T \sim g(\Theta') } \left[ G_T ( \tau_T ) \nabla_{ \bbTheta_i } \log p_{ g(\Theta')} ( \tau_T ) \right] \\
& \quad - \mathbb{E}_{ \tau_T \sim g(\Theta) } \left[ G_T ( \tau_T ) \nabla_{ \bbTheta_i } \log p_{ g(\Theta')} ( \tau_T ) \right] \|.
\end{split}
\end{equation*}
We note that the last term is bounded by $L_1 \| \Theta - \Theta' \|$ under Assumption~\ref{assumption:Distribution}, while the first term is bounded by $T^2 L_0 {\rm R} \| \Theta - \Theta' \|$ under Assumptions \ref{assumption:Bd} and \ref{assumption:GradPolicy}. Summing up over $i=1,...,D$ and rearranging terms shows that 
$\| \nabla_{\Theta} \bar{G}_V(\Theta) - \nabla_{\Theta} \bar{G}_V(\Theta') \| \leq L_G \| \Theta - \Theta' \|$.

Set $\eta^h = 1/\sqrt{H}$. Together with \eqref{eq:moreau_bd}, for sufficiently large $H$, the above shows that in $H$ iterations, the projected TLRPG method finds an ${\cal O}(1/\sqrt{H})$-stationary solution to \eqref{eq::pg_cost}.

\end{proof}

\section{Numerical Experiments}

The aim of this section is to illustrate empirically that low-rank tensor policies can attain similar returns to NN-based methods while demanding fewer resources. More specifically, we show that low-rank tensor policies require less amount of parameters and time to converge. We assess this claim by: i) thoroughly investigating diverse policy-based algorithmic configurations across various classical control problems, followed by ii) an examination of a (more realistic) higher-dimensional wireless communications setup. The code, along with a more detailed description of the experiments, is available in the GitHub repository linked to this paper \cite{rozada2024repo}.

\subsection{Classical Control Problems}

In order to evaluate the performance of the proposed methods we have selected two well-known control problems implemented in the toolkit OpenAI Gym \cite{brockman2016openai}:

\begin{itemize}
    \item \emph{MountainCar}, where a car tries to reach the top of a hill. The state-action space has $3$ dimensions, and the agent attains a positive reward upon reaching the goal, but incurs a penalty proportional to the fuel expended.
    \item \emph{Pendulum}, where the goal is to keep a pendulum pointing upwards. The state-action space has $4$ dimensions, and the reward function weights the position of the pendulum and the torque developed to control it.
\end{itemize}

The described environments admit a continuous and a discrete action-space version.

\noindent \textbf{Empirical low-rankness.} Firstly, we have verified our assertion regarding the low-rank nature of parameter tensors empirically. To that end, we have trained a NN to address both the discrete and continuous versions of the Pendulum problem using PPO with softmax and Gaussian policies, respectively. Subsequently, we have discretized the state space and employed the learned policies to estimate the parameters across all discretized states. These parameters were then collected into a tensor of Gaussian means of size $100 \times 100 \times 100$ and a tensor of softmax logits of size $100 \times 100 \times 100 \times 3$. The maximum possible rank of the former is $K_\text{max}=10,000$, while for the latter is $K_\text{max}^{'}=30,000$. Following this, we have decomposed the tensors using the PARAFAC decomposition for various ranks. To evaluate the approximation performance, we have computed the Normalized Frobenius Error ($\mathrm{NFE}$) between a given tensor $\tenbX$ and its approximation $\check{\tenbX}$, which is defined as:
\begin{equation}
    \mathrm{NFE}={||\tenbX - \check{\tenbX}||_F}\big/{||\tenbX||_F}.
    \label{eq:NFE}
\end{equation}

The results are depicted in Fig. \ref{fig::fig_low_rank}. As observed, $\mathrm{NFE}$ decreases sharply as the rank increases. This indicates that of most of the energy is concentrated in the first factors. Consequently, relatively low-rank approximations of the parameter tensors closely match the original discretized values estimated by NN methods in terms of Frobenius norm. Remarkably, despite the potentially high maximum rank of both tensors, the tensor of Gaussian means exhibits a negligible  $\mathrm{NFE}$ when the rank is set to $5$, with the tensor of softmax logits exhibiting a similar error when  the rank is set to $20$. In plain words, NN methods learn low-rank representations of the policies, even without explicitly imposing it. This underscores the potential of low rank for designing parsimonious policies.
Interestingly, this finding is in agreement with the observations in \cite{rozada2024tensor}, which reports that the VFs of many classical control problems, when arranged in a tensor form, exhibit a low-rank.

\begin{figure}
    \centering
    \includegraphics[width=7cm]{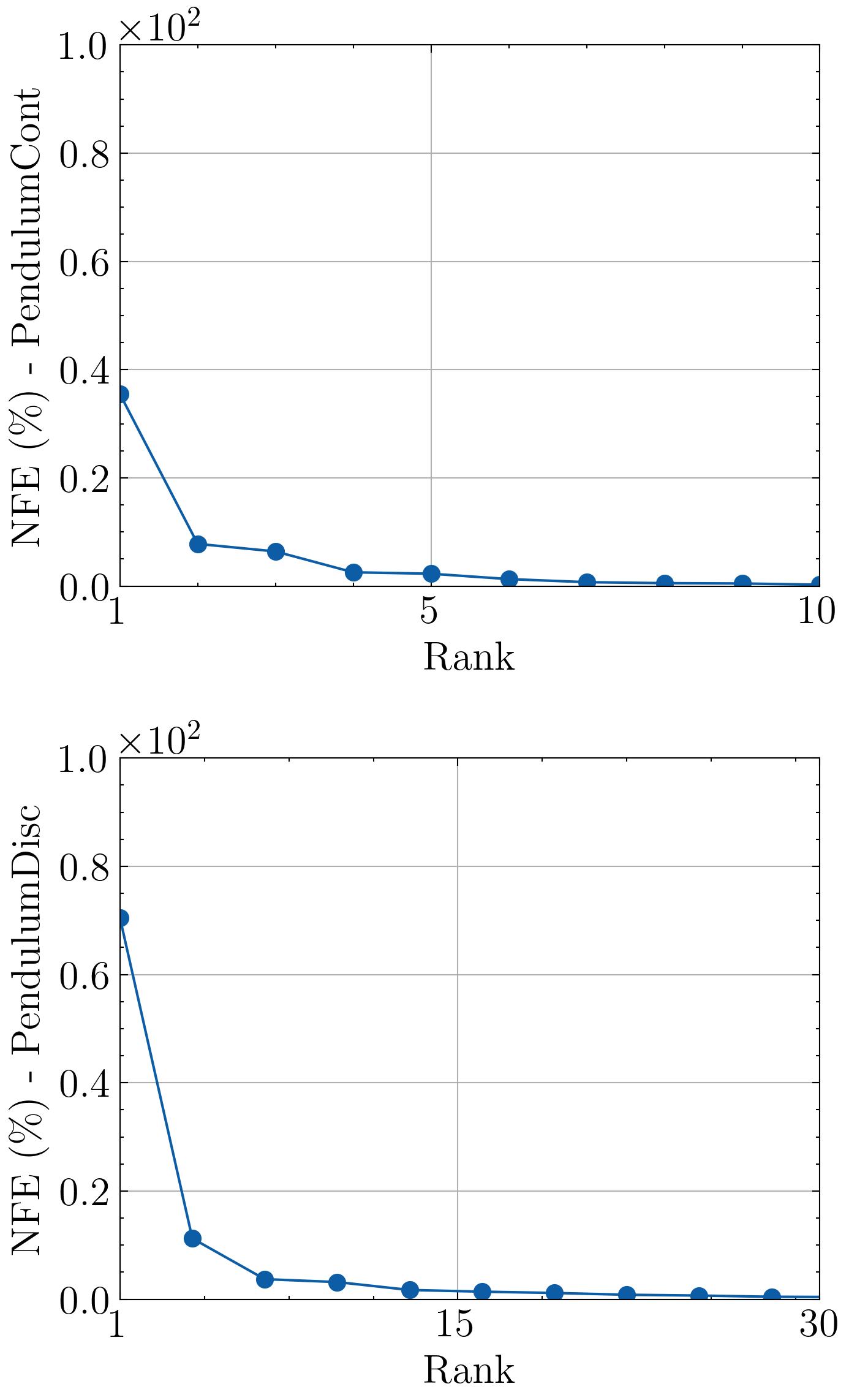}
    \caption{$\mathrm{NFE}$ between the tensor of Gaussian and softmax parameters obtained via NN-based PPO and their low-rank PARAFAC decomposition in the continuous and discrete variants of the Pendulum environment. The error decreases as the rank of the approximation increases.}
    \label{fig::fig_low_rank}
\end{figure}

\noindent \textbf{Continuous action spaces.} We have used Gaussian policies for dealing with continuous action spaces. We have compared the proposed tensor low-rank policies against their NN-based counterparts. We have also considered a baseline linear policy model with radial basis function (RBF) features \cite{buhmann2000radial}. Specifically, we have compared TLRAC, as described in Algorithm \ref{alg::ac}, with a linear-based AC algorithm (RBF-AC), and a NN-based AC algorithm (NN-AC); TRTLRPO, as described in Algorithm \ref{alg::trpo}, against a linear-based TRPO algorithm (RBF-TRPO), and a NN-based TRPO algorithm (NN-TRPO); and PTLRPO, as described in Algorithm \ref{alg::ppo}, against a linear-based PPO algorithm (RBF-PPO), and a NN-based PPO algorithm (NN-PPO). To assess the convergence speed of the proposed algorithms we have measured the return, or cumulative reward, per episode. Convergence is then quantified by how rapidly the algorithm achieves high returns. We have tuned the rank $K$ of our tensor low-rank policies and the hyper-parameters of the NNs via exhaustive search to identify the most compact models, in terms of parameters, capable of yielding high returns. For fairness in comparison, we have set $\sigma$, which controls the exploration in Gaussian policies, and the hyper-parameters that control the computational budget, such as the number of trajectories for the update $U$, to the same values across both tensor low-rank and NN policies. We have conducted $100$ experiments of each combination of setup-algorithm, and the median returns are depicted in Fig. \ref{fig::fig_control_problems_cont}.

While linear methods involve fewer parameters, tensor low-rank methods exhibit significantly faster convergence in most setups. Moreover, tensor low-rank policies exhibit superior performance in terms of achieved returns. Tensor-based methods demand fewer parameters than their NN counterparts across all scenarios. This parameter efficiency translates into superior convergence speed for our proposed algorithms. Remarkably, despite the efficiency gains, we do not significantly compromise on return, with only one exception: NN-TRPO achieves marginally higher returns than TRTLRPO in the PendulumCont. Moreover, our analysis reveals another advantage of tensor low-rank models. The variance in return achieved by TLRAC is notably smaller than that of NN-AC in the PendulumCont environment. 

\begin{figure*}[!t]
    \centering
    \includegraphics[width=18cm]{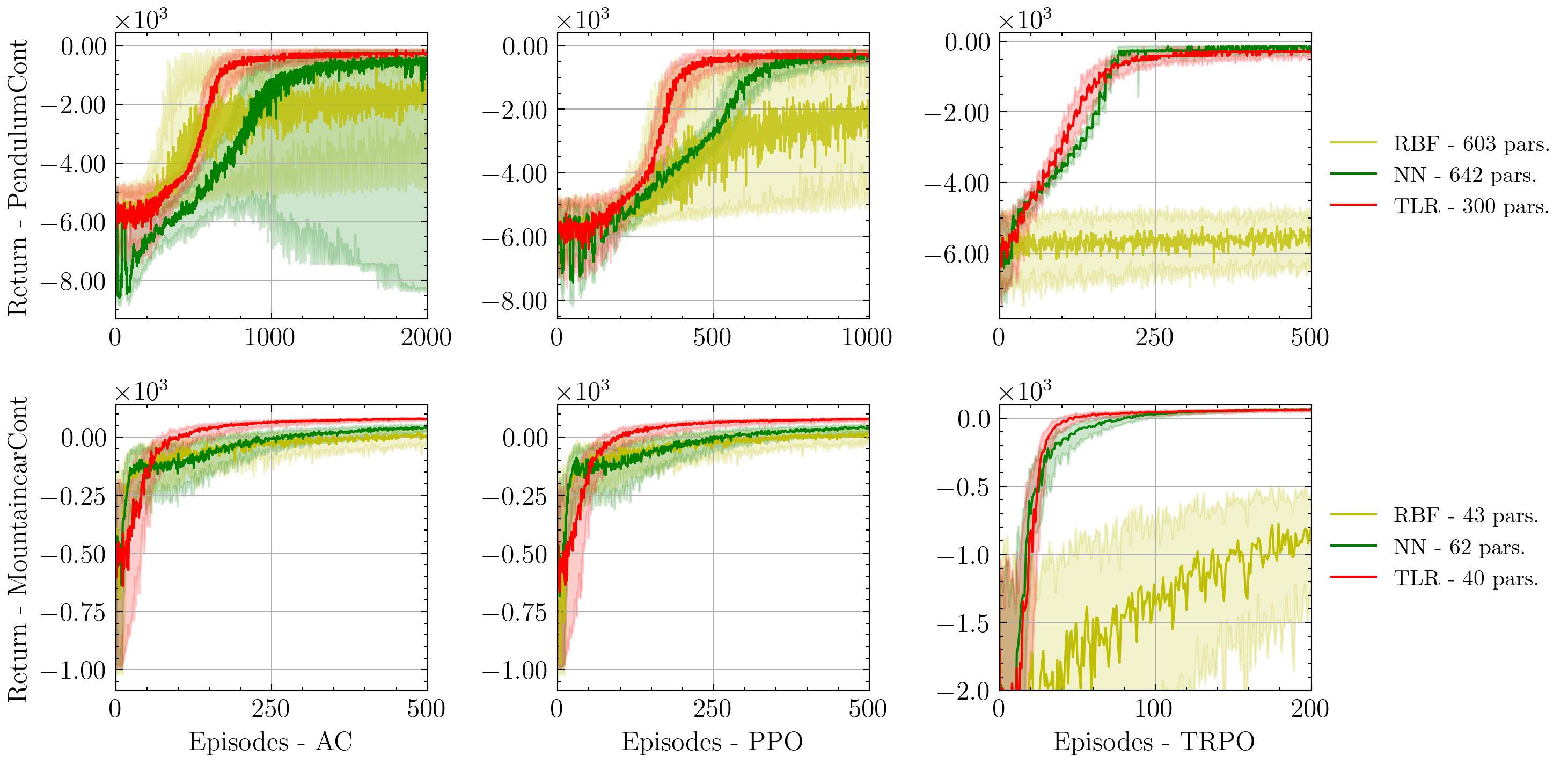}
    \caption{Median return per episode of tensor low-rank policies (TLR) against NN-based policies (NN) across $100$ experiments in various continuous action-space RL problems and algorithmic setups. The displayed confidence interval are the interquartile ranges. TLR policies converge faster than NN-based policies in almost all the scenarios. Furthermore, they require less parameters.}
    \label{fig::fig_control_problems_cont}
\end{figure*}

\noindent \textbf{Discrete action spaces.} To deal with the discrete action-space variants of the previously introduced environments, we have used softmax policies. Again, we have compared TLRAC against RBF-AC and NN-AC, TRTLRPO against RBF-TRPO and NN-TRPO, and PTLRPO against RBF-PPO and NN-PPO. We have used the return per episode as the main metric, and tuned the hyper-parameters via exhaustive search. Also, we have set the temperature parameter $\beta$ and the hyper-parameters that control the computational budget to the same values, and we have conducted $100$ experiments of each combination of setup-algorithm. The median returns are shown in Fig. \ref{fig::fig_control_problems_disc}.

The results are consistent with the previous experiment. Tensor low-rank policies routinely outperform linear methods in terms of convergence speed and achieved return. Furthermore, they not only demand fewer parameters than NN-based algorithms but converge faster as well, except for the PendulumDisc environment, where NN-TRPO demonstrates faster convergence than TRTLRPO. Regarding the return, the only case where NN-PPO achieves marginally higher returns than PTLRPO is the PendulumDisc environment. Again, we observe that the variance in the return achieved by TLRAC is similar to that of NN-AC, except for the MountaincarDisc environment, where our method exhibits a clearly smaller variance. These findings highlight the potential of multi-linear tensor low-rank policy models. They strike a balance, being as parameter-efficient as linear methods yet as effective as NNs in terms of returns. This underscores the relevance of low-rank models in RL scenarios.

\begin{figure*}[!t]
    \centering
    \includegraphics[width=18cm]{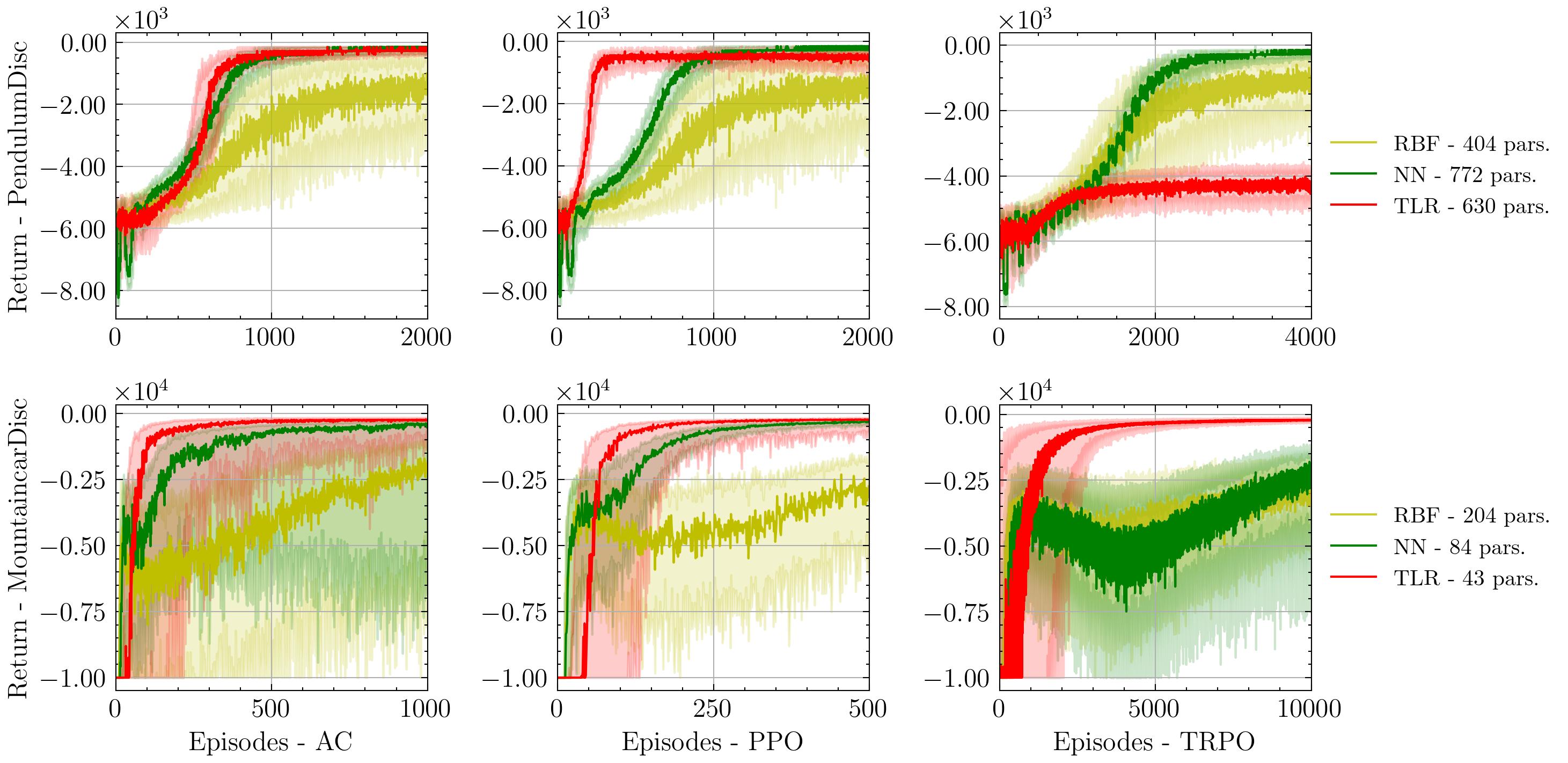}
    \caption{Median return per episode of tensor low-rank policies (TLR) against NN-based policies (NN) across $100$ experiments in various discrete action-space RL problems and algorithmic setups. The displayed confidence interval are the interquartile ranges. TLR policies converge faster than NN-based policies in almost all the scenarios. Furthermore, they require less parameters.}
    \label{fig::fig_control_problems_disc}
\end{figure*}

\subsection{Wireless Communications Setup}

This section examines a wireless communications setup with opportunistic multiple-access. This is not only a more realistic setup, but also a larger problem that helps us to assess how well our methods scale. Specifically, a single agent dispatches packets to an access point using two orthogonal channels ($C=2$). The packets, which arrive every $T=10$ time-slots, accumulate in a queue. Access is opportunistic, and the channels may often be occupied. Sending packets consumes energy from a battery, that recharges over time through energy harvesting. The corresponding MDP encompasses a state space of  $D_\ccalS=6$ dimensions: the fading intensity and occupancy status across the two channels, battery energy levels, and queue packet counts. Each time-step, the agent determines the power allocation for the duo of channels, thus defining an action space of $D_\ccalA=2$ continuous dimensions. Transmission rates follow Shannon's capacity formula, and we establish that 80 \% of the packets are lost when channels are occupied. The reward function accounts for the battery level (positively weighted) and the queue size (negatively weighted), compelling the agent to trade-off transmission throughput and battery levels. More details can be found in \cite{rozada2024repo}.

Given the continuous and multi-dimensional nature of the action space, we have used a multivariate Gaussian policy. We focused on comparing PTLRPO against NN-PPO, due to the widespread utilization of PPO as a policy-based algorithm. Again, we have used returns per episode to evaluate the proposed algorithms, which were tuned for compactness through exhaustive search. Median returns are depicted in Fig. \ref{fig::fig_wireless} based on $100$ experiments. In essence, PTLRPO not only requires fewer parameters than NN-PPO but also demonstrates superior empirical convergence speed. Remarkably, despite the modest difference in parameter count between PTLRPO and NN-PPO, the former achieves steady-state much more rapidly. This shows that by capitalizing on the structure of the MDP via low-rank methods, PTLRPO enables faster and more stable training, thereby reducing the computational resources needed for learning.

\begin{figure}
    \centering
    \includegraphics[width=8.5cm]{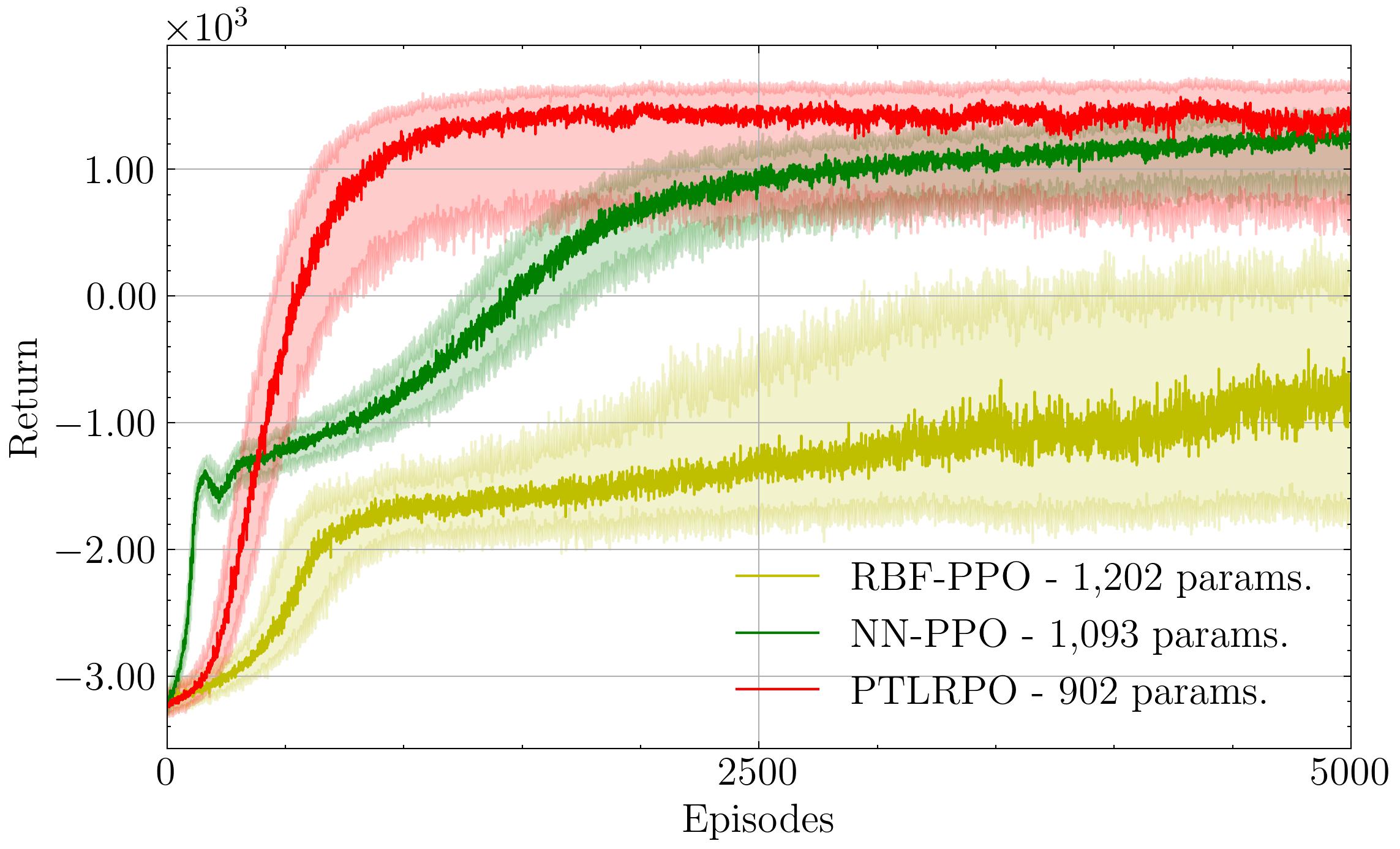}
    \caption{Median return per episode of tensor low-rank policies (TLR) against NN-based policies (NN) across $100$ experiments in the wireless communications setup. The displayed confidence interval are the interquartile ranges. PTLRPO converges faster than NN-PPO while requiring less parameters.}
    \label{fig::fig_wireless}
\end{figure}

\section{Conclusions}

This paper introduced multilinear tensor \emph{low-rank} policy models in the context of policy-based RL. Specifically, it harnessed tensor-completion techniques to approximate Gaussian and softmax policies via multi-linear mappings in policy-gradient RL setups. Within this framework, four tensor \emph{low-rank} algorithms have been proposed: tensor low-rank policy gradient (TLRPG), tensor low-rank actor-critic (TLRAC), trust-region tensor low-rank policy optimization (TRTLRPO), and proximal tensor low-rank policy optimization (PTLRPO). The proposed algorithms are i) easy to adjust since they involve tuning one single model hyper-parameter, the rank of the tensor; ii) convergent under mild conditions as demonstrated theoretically; and iii) potentially faster than their NN-based counterparts as showed through numerical experiments in several standard benchmarks. By capitalizing on the inherent low-rank structure of several RL setups, these approaches aim to mitigate issues commonly encountered in NN-based methods. In a nutshell, building on the fact that low rank is prevalent in RL, this paper put forth a systematic approach to incorporating low-rank multi-linear structure in policy-based RL.

\appendices 

\section{Verifying the Assumptions} \label{app:assumption}
Notice that under Assumption~\ref{assumption:Bd}, $\sup_{ \tau_T } |G_T(\tau_T)| \leq {\rm R} T$ and we recall that
\begin{equation} \label{eq:gradient_expression}
\begin{split}
& \nabla_\Theta \bar{G}_V(\Theta) = \mathbb{E}_{g \left( \Theta \right)} \Big[ {G_T(\tau_T)} \sum_{t=1}^T \nabla_\Theta \text{log} \pi_\Theta(A_t | S_t) \Big] \\
& = \int  G_T( \tau_T ) \nabla \log p_{ g(\Theta) } ( \tau_T ) \, p_{ g(\Theta) }(\tau_T) \, d \tau_T .
\end{split}
\end{equation}
As such, it suffices to prove that $\nabla_{\Theta} \log \pi_\Theta ( A_t^u | S_t^u )$ is Lipschitz and bounded to verify the required assumptions in Sec.~\ref{subsec:analysis}. Hereafter, we assume that $\Theta$ is bounded such that for any $\Theta = \{ \bbTheta_d \}_{d=1}^D \in \ccalO$, $\max_{d=1,...,D} \| \bbTheta_d \|_{\infty} \leq B$. 
Note that this implies $|[\tenbTheta ]_{\bbi_s}| \leq KB^D$.
We observe the following cases. 

For the \emph{Gaussian policy}, we recall that 
\[
\frac{\partial{\text{log} \; \pi_{\Theta } (a | s)}}{\partial{[\bbTheta_d]_{i, k}}} = \frac{a - [\tenbTheta ]_{\bbi_s}}{\sigma^2} \left( \prod_{j = 1 \neq d}^D [\bbTheta_j]_{i^s_d, k}  \right) \ccalI_{i = i^s_d}.
\]
Assume that the actions are always bounded with $|a| \leq A$ for any $a \in {\cal A}$, we can verify 
verify Assumption~\ref{assumption:StocGrad} as 
\begin{equation}
\left| \frac{\partial{\text{log} \; \pi_{\Theta } (a | s)}}{\partial{[\bbTheta_d]_{i, k}}} \right| \leq \frac{ A + KB^D }{ \sigma^2 } B^{D-1},
\end{equation}
which leads to ${\tt G}^2 < \infty$.
Also, Assumption~\ref{assumption:GradPolicy} can be verified by noting that for any $\Theta, \Theta'$
\begin{align}
\nonumber
& \left| \frac{\partial{\text{log} \; \pi_{\Theta } (a | s)}}{\partial{[\bbTheta_d]_{i, k}}}
- \frac{\partial{\text{log} \; \pi_{\Theta' } (a | s)}}{\partial{[\bbTheta_d]_{i, k}}} \right| \leq \ccalI_{i = i^s_d} \left(  \frac{ B^{2D-2}}{ \sigma^2 } \right. \\
\label{eq::ass_2_gaussian}
& \quad \quad \left. + \frac{ (A + KB^D ) B^{D-2} }{ \sigma^2 } \right) \sum_{d=1}^D \big| [\bbTheta_d]_{i^s_d,k} - [\bbTheta_d']_{i^s_d,k} \big| 
\end{align}

This leads to $L_0 < \infty$. 
Finally, to verify Assumption~\ref{assumption:Distribution}, we have 
\begin{align}
\nonumber
& | p_{ g(\Theta) }(\tau_T) - p_{ g(\Theta') }(\tau_T) | \\
& \leq \frac{1}{4} \sum_{t=1}^T | \mu_\Theta(s_t) - \mu_{\Theta'} (s_t) | \leq \frac{D T B^{D-1} }{4} \| \Theta - \Theta' \|
\end{align}
where we used the fact that the Gaussian distribution function is $1/4$-Lipschitz.
Combining with \eqref{eq:gradient_expression} leads to $L_1 < \infty$.

For the \emph{softmax policy}, we similarly recall that
\begin{align}
    \nonumber
    &\frac{\partial{\text{log} \; \pi_{\Theta_\bbz} (a | s)}}{\partial{[\bbTheta_d^\bbz]_{i, k}}} =  \\
    &\sum_{b = 1}^C  \left( \ccalI_{a = b} - \ccalZ([\tenbTheta_\bbz]_{\bbi_s}, \beta)_b \right) \left(\prod_{j = 1 \neq d}^{D+1} [\bbTheta_j^\bbz]_{i^{s, b}_j, k} \right) \ccalI_{i=i^{s, b}_d},  \nonumber
\end{align}
Assume that $|\ccalZ(\cdot)_a| \leq {\tt Z}$ for all $a \in \ccalA$, and recalling $\Theta \in \ccalO$ imply 
\[
\left| \frac{\partial{\text{log} \; \pi_{\Theta_\bbz} (a | s)}}{\partial{[\bbTheta_d^\bbz]_{i, k}}} \right|
\leq C B^D {\tt Z} < \infty
\]
which further leads to Assumption~\ref{assumption:StocGrad}. Assumption~\ref{assumption:GradPolicy} can be verified following the same steps as in the Gaussian policy case [cf. \eqref{eq::ass_2_gaussian}]. Assumption~\ref{assumption:Distribution} can be verified using the fact that the soft-max function is $1$-Lipschitz.

\bibliographystyle{IEEEtran}
\bibliography{journal}

\begin{thebibliography}{10}
\providecommand{\url}[1]{#1}
\csname url@samestyle\endcsname
\providecommand{\newblock}{\relax}
\providecommand{\bibinfo}[2]{#2}
\providecommand{\BIBentrySTDinterwordspacing}{\spaceskip=0pt\relax}
\providecommand{\BIBentryALTinterwordstretchfactor}{4}
\providecommand{\BIBentryALTinterwordspacing}{\spaceskip=\fontdimen2\font plus
\BIBentryALTinterwordstretchfactor\fontdimen3\font minus \fontdimen4\font\relax}
\providecommand{\BIBforeignlanguage}[2]{{%
\expandafter\ifx\csname l@#1\endcsname\relax
\typeout{** WARNING: IEEEtran.bst: No hyphenation pattern has been}%
\typeout{** loaded for the language `#1'. Using the pattern for}%
\typeout{** the default language instead.}%
\else
\language=\csname l@#1\endcsname
\fi
#2}}
\providecommand{\BIBdecl}{\relax}
\BIBdecl

\bibitem{sutton2018reinforcement}
R.~S. Sutton and A.~G. Barto, \emph{Reinforcement Learning: An Introduction}.\hskip 1em plus 0.5em minus 0.4em\relax The MIT Press, 2018.

\bibitem{bertsekas2019reinforcement}
D.~P. Bertsekas, \emph{Reinforcement learning and optimal control}.\hskip 1em plus 0.5em minus 0.4em\relax Athena Scientific, 2019.

\bibitem{silver2016mastering}
D.~Silver \emph{et~al.}, ``Mastering the game of {Go} with deep neural networks and tree search,'' \emph{Nature Publishing Group}, vol. 529, no. 7587, pp. 484--489, 2016.

\bibitem{silver2017mastering}
------, ``Mastering the game of {Go} without human knowledge,'' \emph{Nature Publishing Group}, vol. 550, no. 7676, pp. 354--359, 2017.

\bibitem{brown2020language}
T.~Brown \emph{et~al.}, ``Language models are few-shot learners,'' in \emph{Conf. Neural Information Processing Syst. (NeurIPS)}, vol.~33, 2020, pp. 1877--1901.

\bibitem{bertsekas2000dynamic}
D.~P. Bertsekas, \emph{Dynamic programming and optimal control: Vol. 1}.\hskip 1em plus 0.5em minus 0.4em\relax Athena Scientific, 2000.

\bibitem{bertsekas1996neuro}
D.~P. Bertsekas and J.~N. Tsitsiklis, \emph{Neuro-dynamic programming}.\hskip 1em plus 0.5em minus 0.4em\relax Athena Scientific, 1996.

\bibitem{sutton1999policy}
R.~S. Sutton, D.~McAllester, S.~Singh, and Y.~Mansour, ``Policy gradient methods for reinforcement learning with function approximation,'' in \emph{Conf. Neural Information Processing Syst. (NeurIPS)}, vol.~12, 1999.

\bibitem{paternain2020stochastic}
S.~Paternain, J.~A. Bazerque, A.~Small, and A.~Ribeiro, ``Stochastic policy gradient ascent in reproducing kernel {Hilbert} spaces,'' \emph{IEEE Trans. Auto. Control}, vol.~66, no.~8, pp. 3429--3444, 2020.

\bibitem{chen2021communication}
T.~Chen, K.~Zhang, G.~B. Giannakis, and T.~Ba{\c{s}}ar, ``Communication-efficient policy gradient methods for distributed reinforcement learning,'' \emph{IEEE Trans. Control of Netw. Syst.}, vol.~9, no.~2, pp. 917--929, 2021.

\bibitem{luo2017learning}
Y.~Luo, C.-C. Chiu, N.~Jaitly, and I.~Sutskever, ``Learning online alignments with continuous rewards policy gradient,'' in \emph{IEEE Intl. Conf. Acoust., Speech and Signal Process. (ICASSP)}.\hskip 1em plus 0.5em minus 0.4em\relax IEEE, 2017, pp. 2801--2805.

\bibitem{arulkumaran2017deep}
K.~Arulkumaran, M.~P. Deisenroth, M.~Brundage, and A.~A. Bharath, ``Deep reinforcement learning: A brief survey,'' \emph{IEEE Signal Process. Mag.}, vol.~34, no.~6, pp. 26--38, 2017.

\bibitem{lever2016compressed}
G.~Lever, J.~Shawe-Taylor, R.~Stafford, and C.~Szepesv{\'a}ri, ``Compressed conditional mean embeddings for model-based reinforcement learning,'' in \emph{AAAI Conf. Artificial Intelligence}, vol.~30, no.~1, 2016.

\bibitem{tolstaya2018nonparametric}
E.~Tolstaya, A.~Koppel, E.~Stump, and A.~Ribeiro, ``Nonparametric stochastic compositional gradient descent for {Q-learning} in continuous markov decision problems,'' in \emph{Ann. American Control Conf. (ACC)}.\hskip 1em plus 0.5em minus 0.4em\relax IEEE, 2018, pp. 6608--6615.

\bibitem{hu2022ddpg}
Q.~Hu, S.~Shi, Y.~Cai, and G.~Yu, ``{DDPG}-driven deep-unfolding with adaptive depth for channel estimation with sparse {Bayesian} learning,'' \emph{IEEE Trans. Signal Process.}, vol.~70, pp. 4665--4680, 2022.

\bibitem{yang2019harnessing}
Y.~Yang, G.~Zhang, Z.~Xu, and D.~Katabi, ``Harnessing structures for value-based planning and reinforcement learning,'' in \emph{Intl. Conf. Learning Representations (ICLR)}, 2019.

\bibitem{tsai2021tensor}
K.-C. Tsai, Z.~Zhuang, R.~Lent, J.~Wang, Q.~Qi, L.-C. Wang, and Z.~Han, ``Tensor-based reinforcement learning for network routing,'' \emph{IEEE J. Sel. Topics Signal Process.}, vol.~15, no.~3, pp. 617--629, 2021.

\bibitem{rozada2024tensor}
S.~Rozada, S.~Paternain, and A.~G. Marques, ``Tensor and matrix low-rank value-function approximation in reinforcement learning,'' \emph{IEEE Transactions on Signal Processing}, 2024.

\bibitem{rozada2023matrix}
S.~Rozada and A.~G. Marques, ``Matrix low-rank approximation for policy gradient methods,'' in \emph{IEEE Intl. Conf. Acoust., Speech and Signal Process. (ICASSP)}.\hskip 1em plus 0.5em minus 0.4em\relax IEEE, 2023, pp. 1--5.

\bibitem{eckart1936approximation}
C.~Eckart and G.~Young, ``The approximation of one matrix by another of lower rank,'' \emph{Psychometrika}, vol.~1, no.~3, pp. 211--218, 1936.

\bibitem{markovsky2012low}
I.~Markovsky, \emph{Low rank approximation}.\hskip 1em plus 0.5em minus 0.4em\relax Springer, 2012.

\bibitem{udell2016generalized}
M.~Udell, C.~Horn, R.~Zadeh, S.~Boyd \emph{et~al.}, ``Generalized low rank models,'' \emph{Foundations and Trends{\textregistered} in Machine Learning}, vol.~9, no.~1, pp. 1--118, 2016.

\bibitem{kolda2009tensor}
T.~G. Kolda and B.~W. Bader, ``Tensor decompositions and applications,'' \emph{SIAM Review}, vol.~51, no.~3, pp. 455--500, 2009.

\bibitem{sidiropoulos2017tensor}
N.~D. Sidiropoulos, L.~De~Lathauwer, X.~Fu, K.~Huang, E.~E. Papalexakis, and C.~Faloutsos, ``Tensor decomposition for signal processing and machine learning,'' \emph{IEEE Trans. Signal Process.}, vol.~65, no.~13, pp. 3551--3582, 2017.

\bibitem{agarwal2020flambe}
A.~Agarwal, S.~Kakade, A.~Krishnamurthy, and W.~Sun, ``Flambe: Structural complexity and representation learning of low rank {MDPs},'' in \emph{Conf. Neural Information Processing Syst. (NeurIPS)}, 2020, pp. 20\,095--20\,107.

\bibitem{uehara2021representation}
M.~Uehara, X.~Zhang, and W.~Sun, ``Representation learning for online and offline {RL} in low-rank {MDPs},'' in \emph{Intl. Conf. Learning Representations (ICLR)}, 2021.

\bibitem{barreto2016incremental}
A.~M. Barreto, R.~L. Beirigo, J.~Pineau, and D.~Precup, ``Incremental stochastic factorization for online reinforcement learning,'' in \emph{AAAI Conf. Artificial Intelligence}, vol.~30, no.~1, 2016.

\bibitem{jiang2017contextual}
N.~Jiang, A.~Krishnamurthy, A.~Agarwal, J.~Langford, and R.~E. Schapire, ``Contextual decision processes with low {Bellman} rank are pac-learnable,'' in \emph{Intl. Conf. Machine Learning (ICML)}, vol.~70.\hskip 1em plus 0.5em minus 0.4em\relax JMLR. org, 2017, pp. 1704--1713.

\bibitem{azizzadenesheli2016reinforcement}
K.~Azizzadenesheli, A.~Lazaric, and A.~Anandkumar, ``Reinforcement learning of {POMDPs} using spectral methods,'' in \emph{Conf. Learning Theory}.\hskip 1em plus 0.5em minus 0.4em\relax PMLR, 2016, pp. 193--256.

\bibitem{mahajan2021tesseract}
A.~Mahajan, M.~Samvelyan, L.~Mao, V.~Makoviychuk, A.~Garg, J.~Kossaifi, S.~Whiteson, Y.~Zhu, and A.~Anandkumar, ``Tesseract: Tensorised actors for multi-agent reinforcement learning,'' in \emph{Intl. Conf. Machine Learning (ICML)}.\hskip 1em plus 0.5em minus 0.4em\relax PMLR, 2021, pp. 7301--7312.

\bibitem{sam2023overcoming}
T.~Sam, Y.~Chen, and C.~L. Yu, ``Overcoming the long horizon barrier for sample-efficient reinforcement learning with latent low-rank structure,'' \emph{ACM Measurement and Anal. Comput. Syst.}, vol.~7, no.~2, pp. 1--60, 2023.

\bibitem{shah2020sample}
D.~Shah, D.~Song, Z.~Xu, and Y.~Yang, ``Sample efficient reinforcement learning via low-rank matrix estimation,'' in \emph{Conf. Neural Information Processing Syst. (NeurIPS)}, vol.~33, 2020, pp. 12\,092--12\,103.

\bibitem{rozada2021low}
S.~Rozada, V.~Tenorio, and A.~G. Marques, ``Low-rank state-action value-function approximation,'' in \emph{European Signal Process. Conf. (EUSIPCO)}.\hskip 1em plus 0.5em minus 0.4em\relax IEEE, 2021, pp. 1471--1475.

\bibitem{rozada2023trust}
S.~Rozada and A.~G. Marques, ``Matrix low-rank trust region policy optimization,'' in \emph{IEEE Intl. Wrksp. Computat. Advances Multi-Sensor Adaptive Process. (CAMSAP)}.\hskip 1em plus 0.5em minus 0.4em\relax IEEE, 2023, pp. 1--5.

\bibitem{lee2020optimization}
D.~Lee, N.~He, P.~Kamalaruban, and V.~Cevher, ``Optimization for reinforcement learning: From a single agent to cooperative agents,'' \emph{IEEE Signal Process. Mag.}, vol.~37, no.~3, pp. 123--135, 2020.

\bibitem{williams1992simple}
R.~J. Williams, ``Simple statistical gradient-following algorithms for connectionist reinforcement learning,'' \emph{Machine Learning}, vol.~8, no.~3, pp. 229--256, 1992.

\bibitem{schulman2015high}
J.~Schulman, P.~Moritz, S.~Levine, M.~Jordan, and P.~Abbeel, ``High-dimensional continuous control using generalized advantage estimation,'' \emph{arXiv preprint arXiv:1506.02438}, 2015.

\bibitem{greensmith2004variance}
E.~Greensmith, P.~L. Bartlett, and J.~Baxter, ``Variance reduction techniques for gradient estimates in reinforcement learning.'' \emph{J. Machine Learning Research}, vol.~5, no.~9, 2004.

\bibitem{konda1999actor}
V.~Konda and J.~Tsitsiklis, ``Actor-critic algorithms,'' in \emph{Conf. Neural Information Processing Syst. (NeurIPS)}, vol.~12, 1999, pp. 1008--1014.

\bibitem{kakade2001natural}
S.~M. Kakade, ``A natural policy gradient,'' \emph{Conf. Neural Information Processing Syst. (NeurIPS)}, vol.~14, 2001.

\bibitem{kakade2002approximately}
S.~Kakade and J.~Langford, ``Approximately optimal approximate reinforcement learning,'' in \emph{Intl. Conf. Machine Learning (ICML)}, 2002, pp. 267--274.

\bibitem{schulman2015trust}
J.~Schulman, S.~Levine, P.~Abbeel, M.~Jordan, and P.~Moritz, ``Trust region policy optimization,'' in \emph{Intl. Conf. Machine Learning (ICML)}.\hskip 1em plus 0.5em minus 0.4em\relax PMLR, 2015, pp. 1889--1897.

\bibitem{schulman2017proximal}
J.~Schulman, F.~Wolski, P.~Dhariwal, A.~Radford, and O.~Klimov, ``Proximal policy optimization algorithms,'' \emph{arXiv preprint arXiv:1707.06347}, 2017.

\bibitem{bro1997parafac}
R.~Bro, ``{PARAFAC}. tutorial and applications,'' \emph{Chemometrics and intelligent laboratory systems}, vol.~38, no.~2, pp. 149--171, 1997.

\bibitem{rozada2024tensorB}
S.~Rozada and A.~G. Marques, ``Tensor low-rank approximation of finite-horizon value functions,'' in \emph{ICASSP 2024-2024 IEEE International Conference on Acoustics, Speech and Signal Processing (ICASSP)}.\hskip 1em plus 0.5em minus 0.4em\relax IEEE, 2024, pp. 5975--5979.

\bibitem{amari2012differential}
S.~Amari, \emph{Differential-geometrical methods in statistics}.\hskip 1em plus 0.5em minus 0.4em\relax Springer Science \& Business Media, 2012, vol.~28.

\bibitem{davis2019stochastic}
D.~Davis and D.~Drusvyatskiy, ``Stochastic model-based minimization of weakly convex functions,'' \emph{SIAM Journal on Optimization}, vol.~29, no.~1, pp. 207--239, 2019.

\bibitem{rozada2024repo}
S.~Rozada, ``Online code repository: Tensor low-rank approximation for policy-gradient methods in reinforcement learning,'' \url{https://github.com/sergiorozada12/tensor-low-rank-pg}, 2024.

\bibitem{brockman2016openai}
G.~Brockman, V.~Cheung, L.~Pettersson, J.~Schneider, J.~Schulman, J.~Tang, and W.~Zaremba, ``Open{AI} {G}ym,'' \emph{arXiv preprint arXiv:1606.01540}, 2016.

\bibitem{buhmann2000radial}
M.~D. Buhmann, ``Radial basis functions,'' \emph{Acta numerica}, vol.~9, pp. 1--38, 2000.

\end{thebibliography}

\end{document}